\documentclass{article}

\usepackage[final, nonatbib]{neurips_2021}

\usepackage[utf8]{inputenc} 
\usepackage[T1]{fontenc}    
\usepackage[colorlinks=true,linkcolor=black,anchorcolor=black,citecolor=black,filecolor=black,menucolor=black,runcolor=black,urlcolor=black]{hyperref}       
\usepackage{url}            
\usepackage{doi}
\usepackage{booktabs}       
\usepackage{amsfonts}       
\usepackage{nicefrac}       
\usepackage{microtype}      
\usepackage{xcolor}         
\usepackage{bm}         
\usepackage{graphicx}		
\usepackage{amsmath} 
\usepackage{cite} 
\usepackage{diagbox}
\usepackage{enumitem}
\usepackage{siunitx}
\usepackage{multirow}
\usepackage{tabularx}
\usepackage{hhline}
\usepackage{arydshln}		
\usepackage{xcolor}
\usepackage{mathtools} 
\usepackage{amsthm}			
\usepackage{amssymb}			

\definecolor{lightgreen}{rgb}{.9,1,.9}

\newcolumntype{L}[1]{>{\raggedright\arraybackslash}p{#1}}
\newcolumntype{C}[1]{>{\centering\arraybackslash}p{#1}}
\newcolumntype{R}[1]{>{\raggedleft\arraybackslash}p{#1}}

\theoremstyle{plain} 

\newtheorem{definition}{Definition}
\newtheorem{theorem}{Theorem}
\newtheorem{lemma}{Lemma}

\newtheorem{assumption}{Assumption}

\def\defn{\,\coloneqq\,}
\def\Fix{{\mathsf{Fix}}}
\def\Zer{{\mathsf{Zer}}}
\def\Im{{\mathsf{Im}}}
\def\prox{{\mathsf{prox}}}
\def\proj{{\mathsf{proj}}}
\def\min{\mathop{\mathsf{min}}}

\def\R{\mathbb{R}}

\def\zerobm{{\bm{0}}}

\def\ebm{{\bm{e}}}
\def\hbm{{\bm{h}}}
\def\xbm{{\bm{x}}}
\def\zbm{{\bm{z}}}
\def\ybm{{\bm{y}}}
\def\zbm{{\bm{z}}}
\def\sbm{{\bm{s}}}

\def\Abm{{\bm{A}}}
\def\Dbm{{\bm{D}}}
\def\Pbm{{\bm{P}}}
\def\Fbm{{\bm{F}}}

\def\Lcal{{\mathcal{L}}}
\def\Xcal{{\mathcal{X}}}

\def\Dsf{{\mathsf{D}}}
\def\Gsf{{\mathsf{G}}}
\def\Wsf{{\mathsf{W}}}
\def\Nsf{{\mathsf{N}}}
\def\Isf{{\mathsf{I}}}
\def\Rsf{{\mathsf{R}}}
\def\Ssf{{\mathsf{S}}}
\def\Tsf{{\mathsf{T}}}

\def\xbmast{{\bm{x}^\ast}}
\def\xbmbar{{\overline{\bm{x}}}}
\def\zbmbar{{\overline{\bm{z}}}}
\def\xbmhat{{\widehat{\bm{x}}}}

\title{Recovery Analysis for Plug-and-Play Priors using the Restricted Eigenvalue Condition}

\author{%
  Jiaming Liu\\
  Washington University in St. Louis\\
  \texttt{jiaming.liu@wustl.edu} \\
  \And
  M. Salman Asif\\
  University of California, Riverside\\
  \texttt{sasif@ece.ucr.edu} \\
   \And
  Brendt Wohlberg\\
  Los Alamos National Laboratory\\
  \texttt{brendt@ieee.org} \\
  \And
  Ulugbek S. Kamilov \\
  Washington University in St. Louis\\
  \texttt{kamilov@wustl.edu} \\
}

\begin{document}

\maketitle

\begin{abstract}
  The \emph{plug-and-play priors (PnP)} and \emph{regularization by denoising (RED)} methods have become widely used for solving inverse problems by leveraging pre-trained deep denoisers as image priors.  While the empirical imaging performance and the theoretical convergence properties of these algorithms have been widely investigated, their recovery properties have not previously been theoretically analyzed.  We address this gap by showing how to establish theoretical recovery guarantees for PnP/RED by assuming that the solution of these methods lies near the fixed-points of a deep neural network. We also present numerical results comparing the recovery performance of PnP/RED in compressive sensing against that of recent compressive sensing algorithms based on generative models. Our numerical results suggest that PnP with a pre-trained artifact removal network provides significantly better results compared to the existing state-of-the-art methods.
\end{abstract}

\section{Introduction}


Many imaging problems---such as denoising, inpainting, and super-resolution---can be formulated as an \emph{inverse problem} involving the recovery of an image $\xbmast \in \R^n$ from noisy measurements
\begin{equation}
\ybm = \Abm\xbmast + \ebm \;,
\label{Eq:ForwardProblem}
\end{equation}
where $\Abm \in \R^{m \times n}$ is the measurement operator and $\ebm \in \R^m$ is the noise. \emph{Compressed sensing (CS)}~\cite{Candes.etal2006, Donoho2006} is a related class of
inverse problems that seek to recover a sparse vector $\xbmast$ from $m < n$ measurements. The sparse recovery is possible under certain assumptions on the measurement matrix, such as the \textit{restricted isometry property (RIP)}~\cite{Candes.etal2006} or the \emph{restricted eigenvalue condition (REC)}~\cite{Bickel.etal2009, Wainwright2019}. While traditional CS recovery relies on sparsity-promoting priors, recent work on \emph{compressed sensing using generative models (CSGM)}~\cite{Bora.etal2017} has broadened this perspective to priors specified through pre-trained generative models. CSGM has prompted a large amount of follow-up work on the design and theoretical analysis of algorithms that can leverage generative models as priors for image recovery~\cite{Shah.Hegde2018, Hegde2018, Latorre.etal2019, Hyder.etal2019}.

\emph{Plug-and-play priors (PnP)}~\cite{Venkatakrishnan.etal2013, Sreehari.etal2016} and \emph{regularization by denoising (RED)}~\cite{Romano.etal2017} are two methods related to CSGM that can also leverage pre-trained deep models as priors for inverse problems. However, unlike CSGM, the regularization in PnP/RED is not based on restricting the solution to the range of a generative model, but rather on denoising the iterates with an existing \emph{additive white Gaussian noise (AWGN)} removal method. The effectiveness of PnP/RED has been shown in a number of inverse problems~\cite{Zhang.etal2017a, Metzler.etal2018, Dong.etal2019, Zhang.etal2019, Ahmad.etal2020, Wei.etal2020}, which has prompted researchers to investigate the theoretical properties and interpretations of PnP/RED algorithms~\cite{Chan.etal2016, Meinhardt.etal2017, Buzzard.etal2017, Reehorst.Schniter2019, Ryu.etal2019, Sun.etal2018a, Tirer.Giryes2019, Teodoro.etal2019, Xu.etal2020, Sun.etal2021, Sun.etal2019b, Cohen.etal2020}.


Despite the rich literature on both PnP/RED and CSGM, the conceptual relationship between these two classes of methods has never been formally investigated. In particular, while PnP/RED algorithms enjoy computational advantages over CSGM by not requiring nonconvex projections onto the range of a generative model, they lack theoretical recovery guarantees available for CSGM. In this paper, we address this gap by presenting the first recovery analysis of PnP/RED under the assumptions of CSGM. We show that if a measurement matrix satisfies a variant of REC from~\cite{Bora.etal2017} over the range of a denoiser, then the distance of the PnP solutions to the true $\xbmast$ can be explicitly characterized. We also present conditions under which the solutions of both PnP and RED coincide, providing sufficient conditions for the exact recovery of $\xbmast$ using both methodologies. Our results highlight that the regularization in PnP/RED is achieved by giving preference to images near the \emph{fixed points} of pre-trained deep neural networks. Besides new theory, this paper also presents numerical results directly comparing the recovery performance of PnP/RED against the recent algorithms in compressed sensing from random projections and subsampled Fourier measurements. These numerical results lead to new insights highlighting the excellent recovery performance of both PnP and RED, as well as the benefit of using priors specified as pre-trained \emph{artifact removal (AR)} operators rather than AWGN denoisers.


All proofs and some technical details that have been omitted for space appear in the Supplement, which also provides more background and simulations. The code for our numerical evaluation is available at: \url{https://github.com/wustl-cig/pnp-recovery}.

\section{Background}

\textbf{Inverse problems.} A common approach to estimating $\xbmast$ in~\eqref{Eq:ForwardProblem} is to solve an optimization problem:
\begin{equation}
\label{Eq:OptimizationForInverseProblem}
\min_{\xbm \in \R^n} g(\xbm) + h(\xbm) \quad\text{with}\quad g(\xbm) = \frac{1}{2}\|\ybm-\Abm\xbm\|_2^2 \;,
\end{equation}
where $g$ is a data-fidelity term that quantifies consistency with the observed data $\ybm$ and $h$ is a regularizer that encodes prior knowledge on $\xbm$. For example, a widely-used regularizer in inverse problems is the nonsmooth \emph{total variation (TV)} function $h(\xbm) = \tau \|\Dbm\xbm\|_1$, where $\Dbm$ is the gradient operator and $\tau > 0$ is the regularization parameter~\cite{Rudin.etal1992, Bioucas-Dias.Figueiredo2007, Beck.Teboulle2009a}.

\textbf{Compressed sensing using generative models.} Generative priors have recently become popular for solving inverse problems~\cite{Bora.etal2017}, which typically require solving the optimization problem:
\begin{equation}
\min_{\zbm \in \R^k} \frac{1}{2}\|\ybm - \Abm \Wsf(\zbm)\|_2^2 \;,
\label{Eq:GANMinimization}
\end{equation}
where $\Wsf: \R^k \rightarrow \Im(\Wsf) \subseteq \R^n$ is a pre-trained generative model, such as StyleGAN-2~\cite{Karras.etal2019, Karras.etal2020}. The set $\Im(\Wsf)$ is the image set (or the range set) of the generator $\Wsf$. In the past few years, several algorithms have been proposed for solving this optimization problem~\cite{Shah.Hegde2018, Hegde2018, Latorre.etal2019, Hyder.etal2019}, including the recent algorithms \emph{PULSE}~\cite{Menon.etal2020} and \emph{intermediate layer optimization (ILO)}~\cite{Daras.etal2021} that can recover highly-realistic images. The recovery analysis of CSGM was performed under the assumption that $\Abm$ satisfies the \emph{set-restricted eigenvalue condition (S-REC)}~\cite{Bora.etal2017} over the range of the generative model:
\begin{equation}
\label{Eq:CSGMREC}
\|\Abm\xbm-\Abm\zbm\|_2^2 \geq \mu \|\xbm-\zbm\|_2^2 - \eta \quad \forall \xbm, \zbm \in \Im(\Wsf) \;,
\end{equation}
where $\mu > 0$ and $\eta \geq 0$. S-REC implies that the pairwise distances between vectors in the range of the generative model must be well preserved in the measurement space. It thus broadens the traditional notions of REC and the \emph{restricted isometry property (RIP)} in CS beyond sparse vectors~\cite{Candes.Wakin2008}.

\textbf{PnP and RED.}  PnP~\cite{Venkatakrishnan.etal2013, Sreehari.etal2016} refers to a family of iterative algorithms that are based on replacing the proximal operator $\prox_{\gamma h}$ of the regularizer $h$ within a proximal algorithm~\cite{Parikh.Boyd2014}
by a more general denoiser $\Dsf: \R^n \rightarrow \Im(\Dsf) \subseteq \R^n$, such as BM3D~\cite{Dabov.etal2007} or DnCNN~\cite{Zhang.etal2017}. For example, the widely used \emph{proximal gradient method (PGM)}~\cite{Figueiredo.Nowak2003, Daubechies.etal2004, Bect.etal2004, Beck.Teboulle2009} can be implemented as a PnP algorithm as~\cite{Kamilov.etal2017}
\begin{equation}
\label{Eq:PnPPGM}
\xbm^k = \Tsf(\xbm^{k-1}) \quad\text{with}\quad \Tsf \defn \Dsf(\Isf-\gamma \nabla g) \;,
\end{equation}
where $g$ is the data-fidelity term in~\eqref{Eq:OptimizationForInverseProblem}, $\Isf$ denotes the identity mapping, and  $\gamma > 0$ is the step size. 
Remarkably, this heuristic of using denoisers not associated with any $h$ within a proximal algorithm exhibited great empirical success~\cite{Zhang.etal2017a, Metzler.etal2018, Dong.etal2019, Zhang.etal2019, Ahmad.etal2020, Wei.etal2020} and spurred a great deal of theoretical work on PnP~\cite{Chan.etal2016, Meinhardt.etal2017, Buzzard.etal2017, Reehorst.Schniter2019, Ryu.etal2019, Sun.etal2018a, Tirer.Giryes2019, Teodoro.etal2019, Xu.etal2020, Sun.etal2021}. In particular, it has been recently shown in~\cite{Ryu.etal2019} that, when the residual of $\Dsf$ is Lipschitz continuous, PnP-PGM converges to a point in the fixed-point set of the operator $\Tsf$ that we denote $\Fix(\Tsf)$.

RED~\cite{Romano.etal2017} is a related method, inspired by PnP, for integrating denoisers as priors for inverse problems. For example, the \emph{steepest descent} variant of RED (SD-RED)~\cite{Romano.etal2017} can be summarized as
\begin{equation}
\label{Eq:SDRED}
\xbm^k = \xbm^{k-1}-\gamma \Gsf(\xbm^{k-1})\quad\text{with}\quad \Gsf \defn \nabla g + \tau(\Isf-\Dsf) \;,
\end{equation}
where $\gamma > 0$ is the step size and $\tau > 0$ is the regularization parameter. 
For a locally homogeneous $\Dsf$ that has a strongly passive and symmetric Jacobian, the solution of RED solves~\eqref{Eq:OptimizationForInverseProblem} with $h(\xbm) = (\tau/2)\xbm^\Tsf(\xbm-\Dsf(\xbm))$~\cite{Romano.etal2017, Reehorst.Schniter2019}. Subsequent work has resulted in a number of extensions of RED~\cite{Mataev.etal2019, Sun.etal2019b, Liu.etal2020, Cohen.etal2020, Xie.etal2021}. For example, it has been shown in~\cite{Sun.etal2019b} that, when $\Dsf$ is a nonexpansive operator, SD-RED converges to a point in the zero set of operator $\Gsf$ that we denote as $\Zer(\Gsf)$.

\textbf{Other related work.} While not directly related to our main theoretical contributions, it is worth briefly mentioning other important related families of algorithms that also use deep neural nets for regularizing ill-posed imaging inverse problems (see recent reviews of the area~\cite{McCann.etal2017, Lucas.etal2018, Knoll.etal2020, Ongie.etal2020}). This work is most related to methods that rely on pre-trained priors that are integrated within iterative algorithms, such as a class of algorithms in compressive sensing known as \emph{approximate message passing (AMP)}~\cite{Tan.etal2015, Metzler.etal2016, Metzler.etal2016a, Fletcher.etal2018}. Another related family of algorithms are those based on the idea of \emph{deep unrolling} (for an overview see Section~IV-A in~\cite{Ongie.etal2020}). Inspired by LISTA~\cite{Gregor.LeCun2010}, the unrolling algorithms interpret iterations of a regularized inversion as layers of a CNN and train it end-to-end in a supervised fashion~\cite{Schmidt.Roth2014, Chen.etal2015, Yang.etal2016, zhang2018ista, Aggarwal.etal2019, Liu.etal2021}. Deep image prior \cite{Ulyanov.etal2018} and deep decoder \cite{heckel2018deep} also use neural networks as prior for images; instead of using a pre-trained generative network, they learn the parameters of the network while solving the inverse problem using the available measurements.

\section{Recovery Analysis for PnP and RED}
\label{sec:recoveryanalsysis}

We present two sets of theoretical results for PnP-PGM~\eqref{Eq:PnPPGM} using the measurement model~\eqref{Eq:ForwardProblem} and the least-squares data-fidelity term~\eqref{Eq:OptimizationForInverseProblem}.
We first establish recovery bounds for PnP under a set of sufficient conditions, and then address the relationship between the solutions of PnP and RED. The proofs of all the theorems will be provided in the Supplement. We start by discussing two assumptions that serve as sufficient conditions for our analysis of PnP.

\begin{assumption}
The residual $\Rsf \defn \Isf - \Dsf$ of the operator $\Dsf$ is bounded by $\delta$ and Lipschitz continuous with constant  $\alpha > 0$, which can be written as 
\label{As:PnPLip}
\begin{equation}
\label{Eq:ConstantsAs1}
\|\Rsf(\xbm)\|_2 \leq \delta  \quad\text{and}\quad \|\Rsf(\xbm)-\Rsf(\zbm)\|_2 \leq \alpha \|\xbm-\zbm\|_2, \quad\forall \xbm, \zbm \in \R^n \;.
\end{equation}
\end{assumption}

The rationale for stating Assumption~\ref{As:PnPLip} in terms of the residual $\Rsf$ is based on our interest in \emph{residual} deep neural nets that take a noisy or an artifact-corrupted image at the input and produce the corresponding noise or artifacts at the output. The success of residual learning in the context of image restoration is well known~\cite{Zhang.etal2017}. Prior work has also shown that Lipschitz constrained residual networks yield excellent performance without sacrificing stable convergence~\cite{Ryu.etal2019, Sun.etal2019b}.

Related assumptions have been used in earlier convergence results for PnP \cite{Chan.etal2016,Ryu.etal2019}. For example, one of the most-widely known PnP convergence results relies on the boundedness of $\Dsf$~\cite{Chan.etal2016}. The Lipschitz continuity of the residual $\Rsf$ has been used in the recent analysis of several PnP algorithms in~\cite{Ryu.etal2019}. Both of these assumptions are relatively easy to implement for deep priors. For example, the boundedness of $\Rsf$ can be enforced by simply bounding each output pixel to be within $[0, \nu]$ for images in $[0, \nu]^n \subset \R^n$ for some $\nu > 0$. The $\alpha$-Lipschitz continuity of $\Rsf$ can be enforced by using any of the recent techniques for training Lipschitz constrained deep neural nets~\cite{Ryu.etal2019, Terris.etal2020, Miyato.etal2018, Fazlyab.etal2019}. Fig.~\ref{Fig:visLipz} presents an empirical evaluation of the Lipschitz continuity of $\Rsf$ used in our simulations.

\begin{figure}[t]
\centering\includegraphics[width=13.5cm]{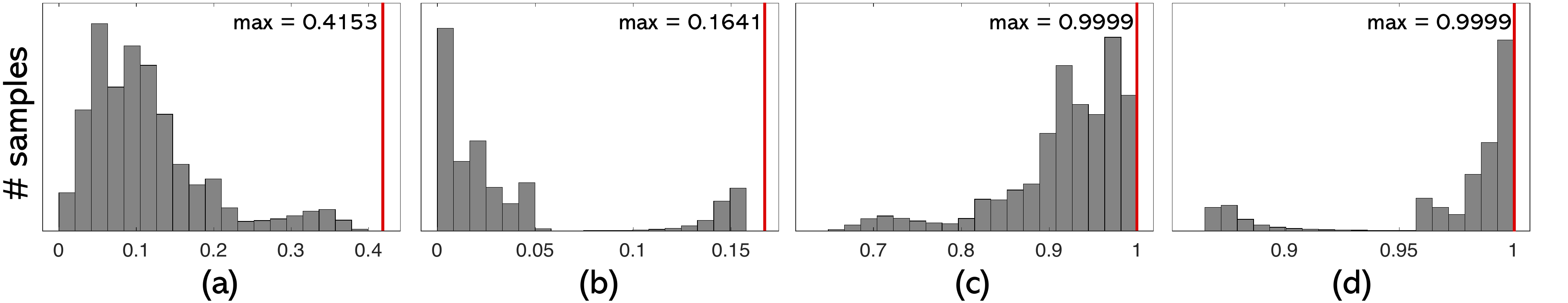}
\caption{~\emph{Empirical evaluation of the Lipschitz continuity of $\Rsf$ and $\Dsf$ used in our simulations and stated in Assumptions~\ref{As:PnPLip} and~\ref{As:LipDenAs}. As described in the main text, we trained two types of Lipschitz constrained networks, where the first simply denoises AWGN and the second removes artifacts specific to the PnP iterations. (a) and (b) show the histograms of $\|\Rsf(\xbm) - \Rsf(\zbm)\|_2/\|\xbm - \zbm\|_2$ for the denoiser and the artifact-removal operator, respectively. (c) and (d) show the histograms of $\|\Dsf(\xbm)-\Dsf(\zbm)\|_2/\|\xbm-\zbm\|_2$ for the same two operators. Note the empirical nonexpansiveness of $\Dsf$ despite the fact that Lipschitz continuity was only imposed on the residual $\Rsf$ during training.}}
\label{Fig:visLipz}
\end{figure}

\begin{assumption}
\label{As:PnPSREC}
The measurement operator $\Abm \in \R^{m \times n}$ satisfies the set-restricted eigenvalue condition (S-REC) over $\Im(\Dsf) \subseteq \R^n$ with $\mu > 0$, which can be written as  
\begin{equation}
\|\Abm\xbm-\Abm\zbm\|_2^2 \geq \mu \|\xbm-\zbm\|_2^2, \quad \forall \xbm, \zbm \in \Im(\Dsf) \;.
\end{equation}
\end{assumption}

S-REC in Assumption~\ref{As:PnPSREC} was adopted from the corresponding assumption for CSGM stated in~\eqref{Eq:CSGMREC}, which establishes a natural conceptual link between those two classes of methods. The main limitation of Assumption~\ref{As:PnPSREC}, which is also present in the traditional RIP/REC assumptions for compressive sensing, lies in the difficulty of verifying it for a given measurement operator $\Abm$. There has been significant activity in investigating the validity of related conditions for randomized matrices for different classes of signals~\cite{Elad2010, Foucart.Rauhut2013, Vershynin2018, Wainwright2019}, including for those synthesized by generative models~\cite{Bora.etal2017, Shah.Hegde2018, Hyder.etal2019, Daras.etal2021}. Despite this limitation, Assumption~\ref{As:PnPSREC} is still conceptually useful as it allows us to relax the strong convexity assumption used in the convergence analysis in~\cite{Ryu.etal2019} by stating that it is sufficient for the strong convexity to hold \emph{only} over the image set $\Im(\Dsf)$ rather than over the whole $\R^n$. This  suggests a new research direction for PnP on designing deep priors with range spaces restricted to satisfy S-REC for some $\Abm$. In the Supplement, we present an empirical evaluation of $\mu$ for the measurement operators used in our experiments by sampling from $\Im(\Dsf)$.

Consider the set $\Fix(\Dsf) \defn \{\xbm \in \R^n : \xbm = \Dsf(\xbm)\}$ of the \emph{fixed points} of $\Dsf$. Note that $\Fix(\Dsf)$ is equivalent to the set $\Zer(\Rsf) \defn \{\xbm \in \R^n : \Rsf(\xbm) = \zerobm\}$ of the \emph{zeros} of the residual $\Rsf = \Isf - \Dsf$. Intuitively, $\Zer(\Rsf)$ consists of all images that produce no residuals, and therefore can be interpreted as the set of all \emph{noise-free} images according to the network. Similarly, when $\Rsf$ is trained to predict artifacts in an image, $\Zer(\Rsf)$ is the set of images that are \emph{artifact-free} according to $\Rsf$. In the subsequent analysis, we use the notation $\Zer(\Rsf)$, but these results can be equivalently stated using $\Fix(\Dsf)$.

We first state the PnP recovery in the setting where there is no noise and $\xbmast \in \Zer(\Rsf)$.
\begin{theorem}
\label{Thm:IdealizedCS}
Run PnP-PGM for $t \geq 1$ iterations under Assumptions~\ref{As:PnPLip}-\ref{As:PnPSREC} for the problem~\eqref{Eq:ForwardProblem} with no noise and $\xbmast \in \Zer(\Rsf)$. Then, the sequence $\xbm^t$ generated by PnP-PGM satisfies
\begin{equation}
\|\xbm^t - \xbmast\|_2 \leq c \|\xbm^{t-1}-\xbmast\|_2 \leq c^t \|\xbm^0-\xbmast\|_2 \;,
\end{equation}
where $\xbm^0 \in \Im(\Dsf)$ and $c \defn (1+\alpha)\max\{|1-\gamma \mu|, |1-\gamma \lambda|\}$ with $\lambda \defn \lambda_{\mathsf{max}}(\Abm^\Tsf\Abm)$.
\end{theorem}

The proof of the theorem is available in the Supplement. Theorem~\ref{Thm:IdealizedCS} extends the theoretical analysis of PnP in~\cite{Ryu.etal2019} by showing convergence to the true solution $\xbmast$ of~\eqref{Eq:ForwardProblem} instead of the fixed points $\Fix(\Tsf)$ of $\Tsf$ in~\eqref{Eq:PnPPGM}. The condition $\xbm^0 \in \Im(\Dsf)$ can be easily enforced by simply passing any initial image through the operator $\Dsf$. One does not necessarily need $\alpha < 1$, for the convergence result in Theorem~\ref{Thm:IdealizedCS}. As shown in~\cite{Ryu.etal2019}, the coefficient $c$ in Theorem~\ref{Thm:IdealizedCS} is less than one if
\begin{equation}
\label{Eq:StepSizeBounds}
\frac{1}{\mu(1+1/\alpha)} < \gamma < \frac{2}{\lambda} - \frac{1}{\lambda(1+1/\alpha)} \;,
\end{equation}
which is possible if $\alpha < 2\mu/(\lambda-\mu)$. Since all PnP algorithms have the same fixed points~\cite{Meinhardt.etal2017, Sun.etal2018a}, our result implies that PnP can exactly recover the true solution $\xbmast$ to the inverse problem, which extends the existing theory in the literature that only shows convergence to $\Fix(\Tsf)$.

We now present a more general result that relaxes the assumptions in Theorem~\ref{Thm:IdealizedCS}.
\begin{theorem}
\label{Thm:RelaxedCS}
Run PnP-PGM for $t \geq 1$ iterations under Assumptions~\ref{As:PnPLip}-\ref{As:PnPSREC} for the problem~\eqref{Eq:ForwardProblem} with $\xbmast \in \R^n$ and $\ebm \in \R^m$. Then, the sequence $\xbm^t$ generated by PnP-PGM satisfies
\begin{equation}
\|\xbm^t - \xbmast\|_2 \leq c \|\xbm^{t-1}-\xbmast\|_2 + \varepsilon \leq c^t \|\xbm^0-\xbmast\|_2 + \frac{\varepsilon(1-c^t)}{(1-c)} \;,
\end{equation}
where $\xbm^0 \in \Im(\Dsf)$ and 
\begin{equation}
\label{Eq:PnPCSBound}
\varepsilon \defn (1+c) \left[\left(1+2\sqrt{\lambda/\mu}\right) \|\xbmast-\proj_{\Zer(\Rsf)}(\xbmast)\|_2 + 2/\sqrt{\mu} \|\ebm\|_2 + \delta(1+1/\alpha) \right]
\end{equation}
and $c \defn (1+\alpha)\max\{|1-\gamma \mu|, |1-\gamma \lambda|\}$ with $\lambda \defn \lambda_{\mathsf{max}}(\Abm^\Tsf\Abm)$.
\end{theorem}

Theorem~\ref{Thm:RelaxedCS} extends Theorem~\ref{Thm:IdealizedCS} by allowing $\xbmast$ to be outside of $\Zer(\Rsf)$ and extends the analysis in~\cite{Bora.etal2017} by considering operators $\Dsf$ that do not necessarily project onto the range of a generative model. In the error bound $\varepsilon$, the first two terms are the distance of $\xbmast$ to $\Zer(\Rsf)$ and the magnitude of the error $\ebm$, and have direct analogs in standard compressed sensing. The third term is the consequence of the possibility for the solution of PnP not being in the zero-set of $\Rsf$ and one can show that when $\Zer(\Rsf) \cap \Zer(\nabla g) \neq \varnothing$, then the third term disappears. As reported in the Supplement, we empirically verified that the distance of the PnP solution to $\Zer(\Rsf)$ is small for both the denoiser and the artifact-removal operators used in our experiments.

Our final result explicitly relates the solutions of PnP and RED. In order to obtain the result, we need an additional assumption that the denoiser $\Dsf = \Isf - \Rsf$ is nonexpansive. 
\begin{assumption}
\label{As:LipDenAs}
The denoiser $\Dsf$ is nonexpansive
$$\|\Dsf(\xbm)-\Dsf(\zbm)\|_2 \leq \|\xbm-\zbm\|_2, \quad \forall \xbm, \zbm \in \R^n \;.$$
\end{assumption}
This is related but different from Assumption~\ref{As:PnPLip} that assumes  the residual $\Rsf$ is $\alpha$-Lipschitz continuous.

The convergence of SD-RED in~\eqref{Eq:SDRED} to $\Zer(\Gsf)$ can be established for a nonexpansive operator $\Dsf$~\cite{Sun.etal2019b}. In principle, the nonexpansiveness of $\Dsf$ can be enforced during the training of the prior in the same manner as that of the more general Lipschitz continuity. However, the prior in our numerical evaluations is trained to have a contractive residual $\Rsf$ without any explicit constraints on $\Dsf$. As a reminder, the nonexpansiveness of $\Rsf$ is only a necessary (but not sufficient) condition for the nonexpansiveness of $\Dsf$~\cite{Bauschke.Combettes2017}. Despite this fact, our empirical evaluation of the Lipschitz constant of $\Dsf$ in Fig.~\ref{Fig:visLipz} indicates that $\Dsf$ used in our experiments is nonexpansive.

\begin{theorem}
\label{Thm:EquivREDPnP}
Suppose that Assumptions~\ref{As:PnPLip}-\ref{As:LipDenAs} are satisfied and that $\Zer(\nabla g) \cap \Zer(\Rsf) \neq \varnothing$, then PnP and RED have the same set of solutions: $\Fix(\Tsf) = \Zer(\Gsf)$.
\end{theorem}

As a reminder, the solutions of PnP correspond to the fixed-points of the operator $\Tsf$ defined in~\eqref{Eq:PnPPGM}, while those of RED to the zeroes of the operator $\Gsf$ defined in~\eqref{Eq:SDRED}. The assumption that $\Zer(\nabla g) \cap \Zer(\Rsf) \neq \varnothing$ implies that there exist vectors that are noise/artifact free according to $\Rsf$ and consistent with the measurements $\ybm$. While this assumption is not universally applicable to all the inverse problems and priors, it still provides a highly-intuitive sufficient condition for the PnP/RED equivalence. Although the relationship between PnP and RED has been explored in the prior work~\cite{Reehorst.Schniter2019, Cohen.etal2020}, to the best of our knowledge, Theorem~\ref{Thm:EquivREDPnP} is the first to prove explicit equivalence. If one additionally considers PnP-PGM with a step size that satisfies the condition in \eqref{Eq:StepSizeBounds}, then $\Tsf$ is a contraction over $\Im(\Dsf)$, which implies that PnP-PGM converges linearly to its unique fixed point in $\Im(\Dsf)$. The direct corollary of our analysis is that, in the noiseless scenario $\ybm = \Abm\xbmast$ with $\xbmast \in \Zer(\Rsf)$, the image $\xbmast$ is the unique fixed point of both PnP and RED over $\Im(\Dsf)$.

It is worth mentioning that several of our assumptions have been stated in a way that simplifies mathematical exposition, but can alternatively be presented in a significantly weaker form. For example, the boundedness assumption in~\eqref{As:PnPLip}---which is used only in the proof of Theorem~\ref{Thm:RelaxedCS}---does not have to hold everywhere, but only at the fixed points of PnP. Indeed, as can be seen in eq.~(6) of the Supplement, the constant $\delta$ is only used to bound the norm of the residual at the fixed-point of PnP-PGM. Similarly, we do not need the nonexpansiveness of $\Dsf$ in Assumption~\ref{As:LipDenAs} to be true everywhere, but only at the fixed points of RED-SD (see page 4 of the Supplement). 

In summary, our theoretical analysis reveals that the fixed-point convergence of PnP/RED algorithms can be strengthened to provide recovery guarantees when S-REC from CSGM is satisfied. Since PnP/RED algorithms do not require nonconvex projections onto the range of a generative model, they enjoy computational benefits over methods that use generative models as priors. However, the literature on generative models is rich with theoretical bounds and recovery guarantees compared to that of PnP/RED. We believe that our work suggests an exciting new direction of research for PnP/RED by showing that a similar analysis can be carried out for PnP/RED.

\section{Numerical Evaluation}
\label{Sec:NumericalValidation}

Before presenting our numerical results, it is important to note that PnP and RED are well-known methods and it is \emph{not} our aim to claim any algorithmic novelty with respect to them. However, comparing PnP/RED to state-of-the-art \emph{compressed sensing (CS)} algorithms is of interest in the context of our theory. Our goal in this section is thus to both (a) empirically evaluate the recovery performance of PnP/RED and (b) compare their performances relative to widely-used CS algorithms.

We consider two scenarios: \emph{(a) CS using random projections} and \emph{(b) CS for magnetic resonance imaging (CS-MRI)}. In order to gain a deeper insights into performance under subsampling, we use an idealized noise-free setting; however, we expect similar relative performances under noise. For each scenario, we include comparisons with several well-established methods based on deep learning.

We consider two priors for PnP/RED: (i) an AWGN denoiser and (ii) an \emph{artifact-removal (AR)} operator trained to remove artifacts specific to the PnP iterations. We implement both priors\footnote{For additional details and code see the Supplement and the GitHub repository.} using the DnCNN architecture~\cite{Zhang.etal2017}, with its batch normalization layers removed for controlling the Lipschitz constant of the network via spectral normalization~\cite{Miyato.etal2018}. We train the denoiser as a nonexpansive residual network $\Rsf$ that predicts the noise residual from a noisy input image. Thus, $\Rsf$ satisfies the necessary condition for the nonexpansiveness of $\Dsf$. Similar to~\cite{Gilton.etal2021}, we train the AR prior by including it into a~\emph{deep unfolding} architecture that performs PnP iterations. When equipped with spectral normalization~\cite{Miyato.etal2018}, the residual $\Rsf$ of the AR operator still satisfies Lipschitz continuity assumptions and achieves superior performance compared to the denoiser (as corroborated by our results). Our implementation also relies on the scaling strategy from~\cite{Xu.etal2021} for controlling the influence of $\Dsf$ relative to $g$. The reconstruction quality is quantified using the peak signal-to-noise ratio (PSNR) in dB.

\begin{figure}[t]
\centering\includegraphics[width=13.5cm]{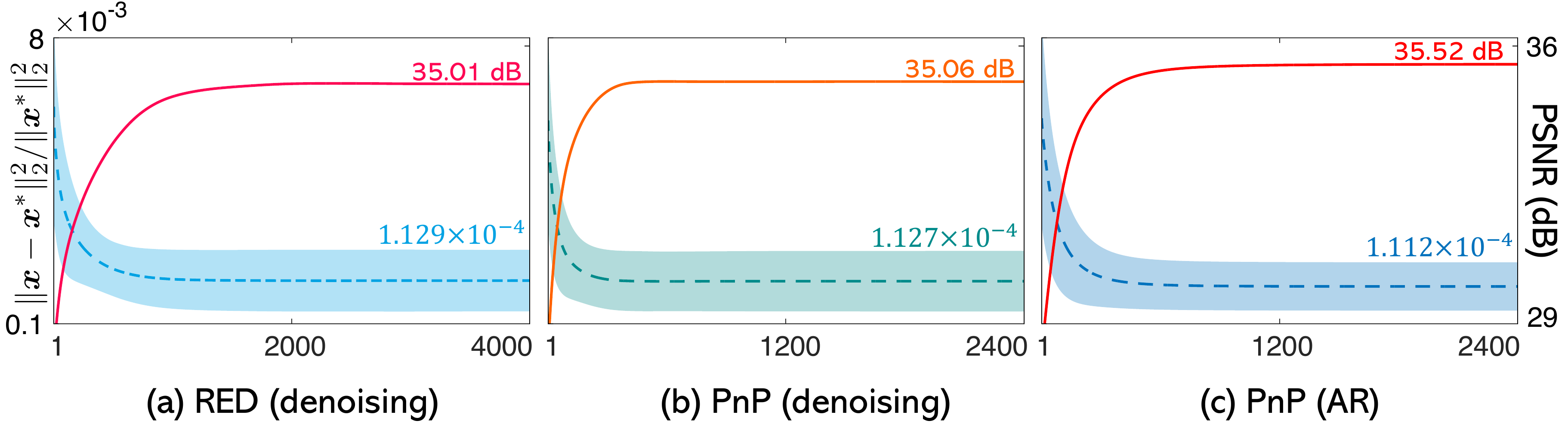}
\caption{~\emph{Empirical evaluation of the convergence of PnP/RED to the true solution $\xbmast$ using a denoiser and an artifact removal (AR) operator. Average normalized distance and PSNR relative to the true solution $\xbmast$ are plotted with the shaded areas representing the range of values attained over all test images. Note the similar recovery performance of PnP and RED, as well as the improvement in performance due to a prior trained to remove artifacts specific to PnP iterations (rather than an AWGN denoiser).}}
\label{Fig:convergence}
\end{figure}

 \begin{table}[t]
\caption{Numerical evaluation of the CS recovery in terms of PSNR (dB) on BSD68 and Set11.}
    \centering
    \renewcommand\arraystretch{1.2}
    {\footnotesize
    \scalebox{1}{
    \begin{tabular*}{13.65cm}{L{70pt}||C{30pt}lC{30pt}lC{30pt}lC{30pt}lC{30pt}lC{30pt}lC{30pt}l}
        \hline
        \multirow{2}{5em}{\diagbox[innerwidth=2.5cm]{\textbf{Method}}{\textbf{CS Ratio}}}& \multicolumn{4}{c|}{\textbf{BSD68} } & \multicolumn{4}{c}{\textbf{Set11} } \\
        \cline{2-9}
        & \textbf{  10\%}  & \textbf{  30\%}  & \textbf{  40\%} & \textbf{  50\%} & \textbf{  10\%} & \textbf{  30\%} & \textbf{  40\%} & \textbf{  50\%}  \\ \hline\hline
        \textbf{TV}       & {24.56}            & {28.61}            & {30.27}            & 31.98    & 24.47 & 30.21           & 32.29   & 34.27 \\
         \textbf{SDA~\cite{Mousavi.etal2015}}       & {23.12}            & {26.38}            & {27.41}            & 28.35    & 22.65 & 26.63           & 27.79   & 28.95 \\
         \textbf{ReconNet~\cite{Kulkarni.etal2016}}       & 24.15            & 27.53           & 29.08           & 29.86    & 24.28 & 28.74           & 30.58   & 31.50  \\
         \textbf{ISTA-Net~\cite{zhang2018ista}}      & {25.02}            & {29.93}            & {31.85}            & 33.61  & 25.80 & 32.91           & 35.36   & 37.43 \\
         \textbf{ISTA-Net$^{+}$~\cite{zhang2018ista}}    & \textcolor{black}{25.33}            & \textcolor{black}{30.34}            & {32.21}            & 34.01   & 26.64  & 33.82           &36.06   & 38.07 \\
         \cdashline{1-9}
         \textbf{RED (denoising)}      & 24.97           &  30.20          & 32.25           &  \textcolor{black}{34.39}  & 27.70 & 35.01           &37.28   & \textcolor{black}{39.26}  \\
         \textbf{PnP (denoising)}       & 25.06          & 30.31           & \textcolor{black}{32.29}           &34.35   & \textcolor{black}{27.76} & \textcolor{black}{35.06}           & \textcolor{black}{37.30}           & 39.21    \\
         \textbf{PnP (AR)}       & \textcolor{black}{\textbf{26.46}}           &  \textcolor{black}{\textbf{31.33} }            & \textcolor{black}{\textbf{ 33.18} }            & \textcolor{black}{\textbf{34.92} }   & \textcolor{black}{\textbf{ 28.98} } & \textcolor{black}{\textbf{35.53}}           &  \textcolor{black}{\textbf{$\:$37.34} }            & \textcolor{black}{\textbf{39.29} }   \\\hline
    \end{tabular*}}
    }
\label{Tab:table1}
\end{table}

\subsection{Reconstruction of Natural Images from Random Projections}
\label{Sec:ValidationNature}

We adopt a simulation setup widely-used in the CS literature, in which non-overlapping $33 \times 33$ patches of an image are measured using the same $m\times n$ random Gaussian matrix $\Abm$, whose rows have been orthogonalized~\cite{Kulkarni.etal2016, zhang2018ista}. The patches are vectorized to $n=1089$-length vectors $\xbm^*$. The training data for the denoiser is generated by adding AWGN to the images from the BSD500~\cite{Martin.etal2001} and DIV2K datasets~\cite{Agustsson.etal2017}. We pre-train several deep models as denoisers for $\sigma \in [1, 15] $, using $\sigma$ intervals of 0.5, and use the denoiser achieving the best PSNR value in each experiment. We use the same set of 91 images as in~\cite{Kulkarni.etal2016} to train the AR operators that are implemented on individual image patches at a time for the CS ratios $(m/n)$ of $\{10\%, 30\%,40\%, 50\%\}$. In order to overcome the block-artifacts in the recovered images, we implement PnP and RED regularizers over the entire image while still using the per-patch measurement model for $\nabla g$.

Our first numerical study in Fig.~\ref{Fig:visLipz} evaluates the Lipschitz continuity of our pre-trained denoisers and the AR operators by following the procedure in~\cite{Ryu.etal2019}.  We use the residual $\Rsf$ and its corresponding operator $\Dsf = \Isf - \Rsf$ and plot the histograms of $\alpha_1 = \|\Rsf(\xbm) - \Rsf(\zbm)\|_2/\|\xbm - \zbm\|_2$ and $\alpha_2 = \|\Dsf(\xbm) - \Dsf(\zbm)\|_2/\|\xbm - \zbm\|_2$ over 1160 AWGN corrupted image pairs extracted from BSD68. The maximum value of each histogram is indicated by a vertical bar, providing an empirical bound on the Lipschitz constants. Fig.~\ref{Fig:visLipz} confirms empirically that  both $\Rsf$ and $\Dsf$ are contractive operators.

Theorem~\ref{Thm:RelaxedCS} establishes that the sequence of iterates $\xbm^t$ generated by PnP-PGM converges to the true solutions $\xbmast$ up to an error term. Fig.~\ref{Fig:convergence} illustrates the convergence behavior of PnP/RED in terms of $\|\xbm^t-\xbmast\|_2^2/\|\xbmast\|_2^2$ and peak signal-to-noise ratio (PSNR) for CS with subsampling ratio of $30\%$ on Set11.  The shaded areas represent the range of values attained across all test images. The results in Fig.~\ref{Fig:convergence} are consistent with our general observation that the PnP/RED algorithms converge in all our experiments for both types of priors and achieve excellent recovery performance.

\begin{figure}[t]
\centering\includegraphics[width=13.5cm]{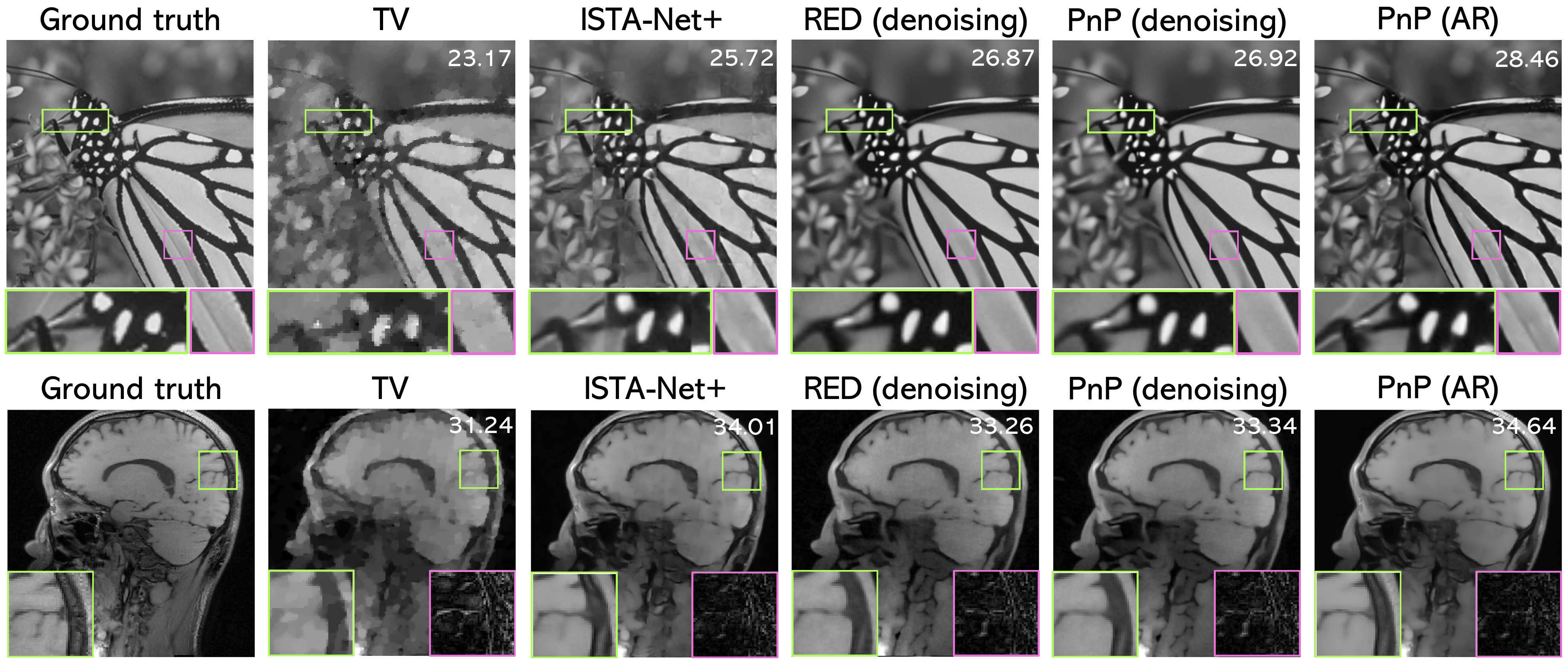}
\caption{~\emph{Visual evaluation of various compressive sensing algorithms at $10\%$ sampling on two imaging problems: (top) reconstruction of \emph{Butterfly} from Set11; (bottom) reconstruction of a brain MR image from its radial Fourier measurements. The pink box in the bottom image provides the error residual that was amplified by $10\times$ for better visualization. Note the similar performance of PnP and RED, as well as the competitiveness of both relative to other methods. Additionally, note the improvement due to the usage of an AR prior instead of an AWGN denoiser within PnP.}}
\label{Fig:visual_nature_mri}
\end{figure}

 \begin{table}[t]
\caption{Average PSNR values for various CS-MRI methods on test images from~\cite{zhang2018ista}.}
    \centering
    \renewcommand\arraystretch{1.2}
    {\footnotesize
    \scalebox{1}{
    \begin{tabular*}{11.6cm}{L{80pt}||C{48pt}lC{48pt}lC{48pt}lC{48pt}l}
        \hline
         \diagbox{\bf Method}{\bf CS Ratio}& \textbf{  10\%}  & \textbf{  20\%}  & \textbf{  30\%} & \textbf{  40\%} & \textbf{~  50\%}  \\ \hline\hline
         \textbf{TV}       & {31.36}            & {35.62}            & {38.41}            & 40.43    & 42.20  \\
         \textbf{ADMM-Net~\cite{Yang.etal2016}}      & {34.19}            & {37.17}            & {39.84}            & 41.56  & 43.00  \\
         \textbf{ISTA-Net$^{+}$~\cite{zhang2018ista}}    & \textcolor{black}{34.65}            & {38.70}            & {40.97}            & 42.65   & 44.12  \\
         \textbf{RED (denoising)}      & 34.37           & 38.63           & 40.94           &42.62   & 44.21  \\
         \textbf{PnP (denoising)}       & 34.56          & \textcolor{black}{38.74}           & \textcolor{black}{41.06}           & \textcolor{black}{42.73}   & \textcolor{black}{ 44.24}  \\
         \textbf{PnP (AR)}       & \textcolor{black}{\textbf{35.21}}           &  \textcolor{black}{\textbf{39.05} }            & \textcolor{black}{\textbf{ 41.28} }            & \textcolor{black}{\textbf{42.96} }   & \textcolor{black}{\textbf{ 44.47} }  \\\hline
    \end{tabular*}}
    }
\label{Tab:table3}
\end{table}

We also report the average PSNR values obtained by five baseline CS algorithms, namely TV~\cite{Beck.Teboulle2009a}, SDA~\cite{Mousavi.etal2015}, ReconNet~\cite{Kulkarni.etal2016}, ISTA-Net~\cite{zhang2018ista} and ISTA-Net${+}$~\cite{zhang2018ista}. TV is an iterative methods that does not require training, while the other four are all deep learning-based methods that have publicly available implementations. The numerical results on Set11 and BSD68 with respect to four measurement rates are summarized in Table~\ref{Tab:table1}. We observe that the performances of PnP and RED are nearly identical to one another. The result also highlights that PnP using the AR prior provides the best performance\footnote{We did not use RED with the AR prior in our experiments since it is expected to closely match PnP.}
compared to all the other methods, outperforming PnP using the AWGN denoiser by at least 0.57 dB on BSD68. Fig.~\ref{Fig:visual_nature_mri} (top) shows visual examples for an image from Set11. Note that both PnP and RED yield similar visual recovery performance. The enlarged regions in the image suggest that PnP (AR) better recovers the fine details and sharper edges compared to other methods.

\subsection{Image reconstruction in Compressed Sensing MRI}
\label{Sec:CSMRI}
MRI is a widely-used medical imaging technology that has known limitations due to the low speed of data acquisition. CS-MRI~\cite{Lustig.etal2007, Lustig.etal2008} seeks to recover an image $\xbmast$ from its sparsely-sampled Fourier measurements. We simulate a single-coil CS-MRI using radial Fourier sampling. The measurement operator $\Abm$ is thus $\Abm=\Pbm\Fbm$, where $\Pbm$ is the diagonal sampling matrix and $\Fbm$ is the Fourier transform.

\begin{figure}[t]
\centering\includegraphics[width=13.5cm]{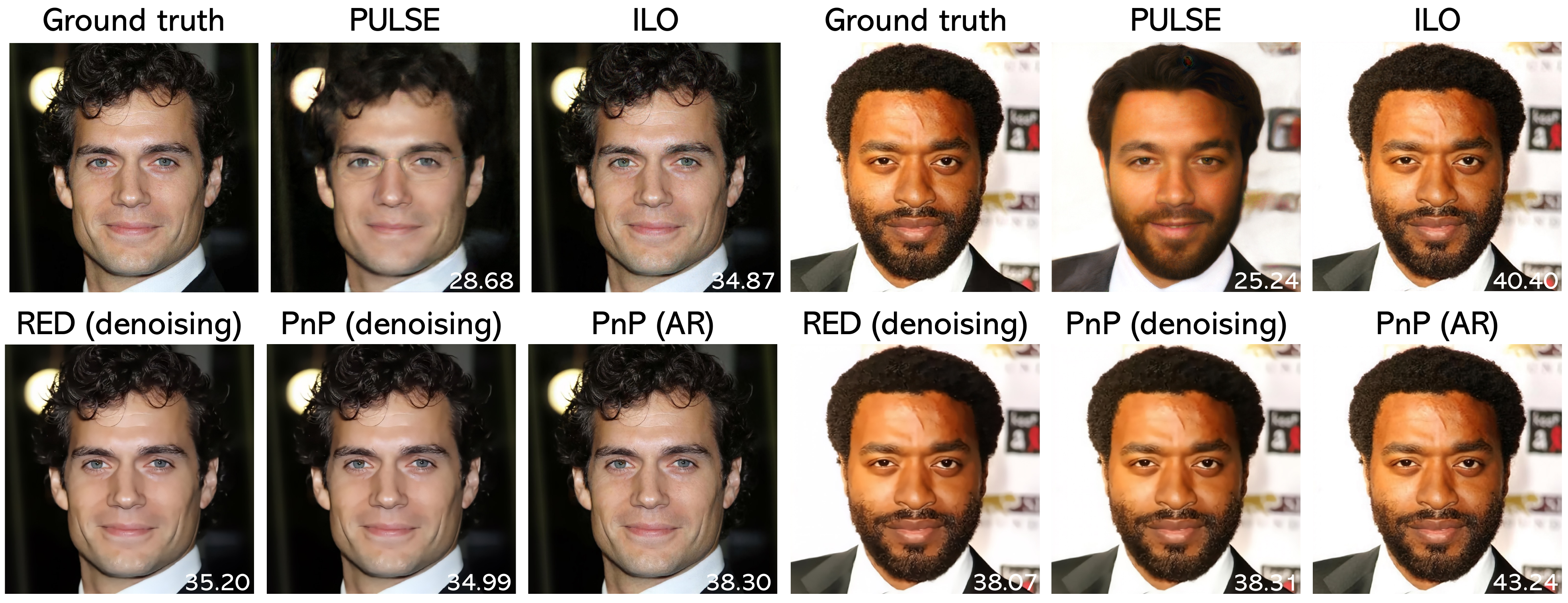}
\caption{~\emph{Visual evaluation of PnP/RED and two methods using generative models as priors on the CelebA HQ~\cite{karras.etal2018} dataset at $10\%$ CS sampling. Note the visual and quantitative similarity of PnP and RED when both are using AWGN denoisers. PnP using an artifact-removal (AR) prior visually matches the performance of ILO based on StyleGAN2, which highlights the benefit of using pre-trained AR operators within PnP. Best viewed by zooming in the display.}}
\label{Fig:vsGM}
\end{figure}

 \begin{table}[t]
\caption{Average PSNR (dB) values for several algorithms on test images from CelebA HQ.}
    \centering
    \renewcommand\arraystretch{1.2}
    {\footnotesize
    \scalebox{1}{
    \begin{tabular*}{11.6cm}{L{80pt}||C{48pt}lC{48pt}lC{48pt}lC{48pt}l}
        \hline
         \diagbox{\bf Method}{\bf CS Ratio}& \textbf{ 10\%}  & \textbf{  20\%}  & \textbf{  30\%} & \textbf{  40\%} & \textbf{  50\%}  \\ \hline\hline
         \textbf{TV}       & {32.13}            & {35.24}            & {37.41}            & 39.35    & 41.29  \\
         \textbf{PULSE~\cite{Menon.etal2020}  }      & {27.45}            & {29.98}            & {33.06}            & 34.25 & 34.77  \\
         \textbf{ILO~\cite{Daras.etal2021}}    & \textcolor{black}{36.15}            & {40.98}            & {43.46}            & 47.89   & 48.21  \\
         \textbf{RED (denoising)}      & 35.46           &\textcolor{black}{41.59}           & 45.65           & \textcolor{black}{48.13}   & 52.17  \\
         \textbf{PnP (denoising)}       & 35.61          & {41.51}           & \textcolor{black}{45.71}           & {48.05}   & \textcolor{black}{ 52.24}  \\
         \textbf{PnP (AR)}       & \textcolor{black}{\textbf{39.19}}           &  \textcolor{black}{\textbf{44.20} }            & \textcolor{black}{\textbf{ 48.66} }            & \textcolor{black}{\textbf{51.32} }   & \textcolor{black}{\textbf{ 53.89} }  \\\hline
    \end{tabular*}}
    }
\label{Tab:table2}
\end{table}

The priors for PnP/RED were trained using the brain dataset from~\cite{zhang2018ista}, where the test set contains 50 slices of $256\times256$ images (i.e., $n = 65536$). We train seven variants of DnCNN, each using a separate noise level from $\sigma \in \{1,1.5, 2, 2.5, 3, 4, 5\}$. Similarly, we separately train the AR operators for different CS ratios $(m/n)$, initializing the weights of the models from the pre-trained denoiser with $\sigma=2$.  For these sets of experiments, we also equipped PnP/RED with \emph{Nesterov acceleration}~\cite{Nesterov2004} for faster convergence. We compare PnP/RED against publicly available implementations of several well-known methods, including TV~\cite{Beck.Teboulle2009a}, ADMM-Net~\cite{Yang.etal2016}, and ISTA-Net$^{+}$~\cite{zhang2018ista}. The last two are deep unrolling methods that train both image transforms and shrinkage functions within the algorithm.

Table~\ref{Tab:table3} reports the results for five CS ratios. The visual comparison can be found in Fig.~\ref{Fig:visual_nature_mri} (bottom). It can been seen that PnP/RED with an AWGN denoiser matches the performance of ISTA-Net$^+$ and outperforms ADMM-Net at higher sampling ratios, while PnP with an AR prior improves over PnP/RED with an AWGN denoiser~\cite{Eslahi.Foi2018}. Note also the similarity of PnP and RED performances.

\subsection{Comparison with generative models on human faces}
\label{Sec:VSCSGM}

We numerically evaluated the recovery performance of PnP/RED in CS against two recent algorithms using generative models: PULSE~\cite{Menon.etal2020} and ILO~\cite{Daras.etal2021}. Similar to the measurement matrix used for grayscale images, we use orthogonalized random Gaussian matrices for sampling image blocks of size $33\times33\times3$. The test images correspond to 15 images randomly selected from CelebA HQ~\cite{karras.etal2018} dataset, each of size $1024 \times 1024$ pixels. We use the DIV2K~\cite{Agustsson.etal2017} and 200 high quality face images from FFHQ dataset~\cite{Karras.etal2019} to train the PnP/RED denoisers for color image denoising at six noise levels corresponding to $\sigma \in \{1, 2, 3, 4, 7, 10\}$.  We use the same training set to train the AR operators for CS ratios of $[10\%, 50\%]$, using the ratio intervals of $10\%$. Similar to CS-MRI, we equipped PnP/RED with Nesterov acceleration. The PSNR comparison between different methods is presented in Table~\ref{Tab:table2}. It can be seen that ILO outperforms PULSE in terms of PSNR, which is consistent with the results in~\cite{Daras.etal2021}. 
Note also how PnP/RED match or sometimes quantitively outperform ILO at high CS ratios, with PnP (AR) leading to significantly better results compared to PnP (denoising). Fig.~\ref{Fig:vsGM} provides visual reconstruction examples. Note the ILO images are sharper compared to PnP/RED with denoisers because ILO uses a state-of-the-art generative model specifically trained on face images. However, PnP (AR) achieves better PSNR and a similar visual quality as ILO.

\section{Conclusion and Future Work}\label{sec:conclusion}

The main goal of this work is to address the theoretical gap between two-widely used classes of methods for solving inverse problems, namely PnP/RED and CSGM. Motivated by the theoretical analysis of CSGM, we used S-REC to establish recovery guarantees for PnP/RED. Our theoretical results provide a new type of convergence for PnP-PGM that goes beyond a simple fixed-point convergence by showing convergence relative to the true solution. Additionally, we show the full equivalence of PnP and RED under some explicit conditions on the inverse problem. While the focus of this work is mainly theoretical, we presented several numerical evaluations that can provide additional insights into PnP/RED and their performance relative to standard methods used in compressed sensing. Empirically, we observed the similarity of PnP/RED in image reconstruction from subsampled random projections and Fourier transform. We also provided additional evidence on the suboptimality of AWGN denoisers compared to artifact-removal operators that take into account the actual artifacts within PnP iterates. 

The work presented in this paper has a certain number of limitations and possible directions for improvement. The main limitation of our theoretical analysis, which is common to all compressive sensing research, is in the difficulty of theoretically verifying S-REC for a given measurement operator. One can also consider the Lipschitz assumptions on $\Rsf$/$\Dsf$ as a limitation, since those can have a negative impact on the recovery. However, our results suggest that even with Lipschitz constrained priors, PnP/RED are competitive with widely-known CS algorithms. While PnP/RED can be implemented using non-Lipschitz-constrained priors, we expect that this will reduce their stability and ultimately hurt their recovery performances. A relatively minor limitation of our simulations is that they were performed without AWGN. One can easily re-run our code by including AWGN and we expect that the relative performances will be preserved for a reasonable amount of noise. We hope that this work will inspire further theoretical and algorithmic research on PnP/RED that will lead to extensions and improvements to our results.

\section{Broader impact}\label{sec:broaderImpacts}

This work is expected to impact the area of imaging inverse problems with potential applications to computational microscopy, computerized tomography, medical imaging, and image restoration. There is a growing need in imaging to deal with noisy and incomplete measurements by integrating multiple information sources, including physical information describing the imaging instrument and learned information characterizing the statistical distribution of the desired image. The ability to accurately solve inverse problems has the potential to enable new technological advances for imaging. These advances might lead to new imaging tools for diagnosing health conditions, understanding biological processes, or inferring properties of complex materials. Traditionally, imaging relied on linear models and fixed transforms (filtered back projection, wavelet transform) that are relatively straightforward to understand. Learning based methods, including PnP and RED, have the potential to enable new technological capabilities; yet, they also come with a downside of being much more complex. Their usage might thus lead to unexpected outcomes and surprising results when used by non-experts. While we aim to use our method to enable positive contributions to humanity, one can also imagine nonethical usage of imaging technology. This work focuses on understanding theoretical properties of imaging algorithms using learned priors, but it might be adopted within broader data science, which might lead to broader impacts that we have not anticipated.

\begin{ack}
Research presented in this article was supported by NSF awards CCF-1813910, CCF-2043134, and  CCF-2046293 and by the Laboratory Directed Research and Development program of Los Alamos National Laboratory under project number 20200061DR.

\end{ack}


\newpage 

\appendix

{\Large\textbf{Supplementary Material}}

\medskip

The mathematical analysis presented in this supplementary document builds on two distinct lines of work: (a) monotone operator theory~\cite{Bauschke.Combettes2017, Ryu.Boyd2016} and (b) compressive sensing using generative models (CSGM)~\cite{Bora.etal2017}. In Section~\ref{Sup:Sec:Theorem1}, we build on past work to prove the convergence of PnP-PGM to the true solution of the inverse problem in the absence of noise. In Section~\ref{Sup:Sec:Theorem2}, we extend the result in Section~\ref{Sup:Sec:Theorem1} to $\xbmast \in \R^n$ and $\ebm \in \R^m$ (i.e., when the signal can be arbitrary and measurements can have noise). In Section~\ref{Sup:Sec:Theorem3}, we show that PnP/RED can have the same set of solutions under some specific conditions. In Section~\ref{Sup:Sec:Backmaterial}, we provide background material useful for our theoretical analysis. Finally, in Section~\ref{Sec:TechnicalDetails}, we provide additional technical details on our implementations and simulations omitted from the main paper due to space.

The algorithmic details of PnP-PGM and SD-RED are summarized in Fig.~\ref{Fig:PnPRED}. It is important to note that it is not our intent to claim any algorithmic novelty in PnP/RED, which are well-known methods. However, there is a strong interest in understanding the theoretical properties of PnP/RED in terms of both recovery and convergence. The main contribution of this work is the development of new theoretical insights into the recovery and convergence of PnP/RED. Finally, our code, including pre-trained denoisers and AR operators, is also included in the supplementary material.

We follow the same notation in the supplement as in the main manuscript. The measurement model corresponds to $\ybm = \Abm\xbmast + \ebm$, where $\xbmast$ is the true solution and $\ebm$ is the noise. The function $g(\xbm) = \frac{1}{2}\|\ybm-\Abm\xbm\|_2^2$ denotes the data-fidelity term. The operator $\Dsf$ denotes the PnP/RED prior, which is implemented via its residual $\Rsf \defn \Isf - \Dsf$. The operator $\Tsf \defn \Dsf(\Isf-\gamma \nabla g)$  denotes the PnP update and $\Gsf \defn \nabla g + \tau \Rsf$ denotes the term used to compute RED updates.

\section{Proof of Theorem 1}
\label{Sup:Sec:Theorem1}

In this section, we prove the first of the main theoretical result in this work, namely the convergence of PnP-PGM to the true solution of the problem $\ybm = \Abm\xbmast$ when $\xbmast \in \Zer(\Rsf)$. Our analysis extends the existing convergence analysis of PnP-PGM from~\cite{Ryu.etal2019}, which proved a linear convergence of the algorithm to $\Fix(\Tsf)$. Here we extend~\cite{Ryu.etal2019} by using the fact that $\xbmast \in \Zer(\Rsf) \cap \Zer(\nabla g)$ and relaxing the assumption of strong convexity in~\cite{Ryu.etal2019} to S-REC over $\Im(\Dsf)$.

\begin{theorem}
\label{Thm:IdealizedCS}
Run PnP-PGM for $t \geq 1$ iterations under Assumptions~1-2 for the problem~(1) of the main paper with no noise and $\xbmast \in \Zer(\Rsf)$. Then, the sequence $\xbm^t$ generated by PnP-PGM satisfies
\begin{equation}
\|\xbm^t - \xbmast\|_2 \leq c \|\xbm^{t-1}-\xbmast\|_2 \leq c^t \|\xbm^0-\xbmast\|_2 \;,
\end{equation}
where $\xbm^0 \in \Im(\Dsf)$ and $c \defn (1+\alpha)\max\{|1-\gamma \mu|, |1-\gamma \lambda|\}$ with $\lambda \defn \lambda_{\mathsf{max}}(\Abm^\Tsf\Abm)$.
\end{theorem}

Suppose all the assumptions for Theorem~1 are true and the step size $\gamma > 0$ is selected in a way that satisfies eq.~(10) in the main paper.
First note that we have assumed that $\xbm^0 \in \Im(\Dsf)$ and we have
\begin{equation*}
\xbm^t = \Tsf(\xbm^{t-1}) = \Dsf(\xbm^{t-1}-\gamma \nabla g(\xbm^{t-1})) \in \Im(\Dsf) \;,
\end{equation*}
which implies that all the PnP-PGM iterates $\{\xbm^t\}$ are in $\Im(\Dsf)$.

Note also the following equivalences
\begin{subequations}
\label{Sup:Eq:SetNotations}
\begin{align}
\label{Sup:Eq:SetNotationsA}&\Zer(\nabla g) = \Fix(\Isf - \gamma \nabla g) =  \{\xbm \in \R^n: \nabla g(\xbm) = \zerobm\} = \{\xbm \in \R^n : \Abm\xbm = \ybm\}\\
\label{Sup:Eq:SetNotationsB}&\Zer(\Rsf) = \Fix(\Dsf) = \{\xbm \in \R^n : \Rsf(\xbm) = \zerobm\} \;,
\end{align}
\end{subequations}
where the first equality in~\eqref{Sup:Eq:SetNotationsA} is due to the following equivalence true for any $\xbm \in \R^n$ and $\gamma > 0$
\begin{equation*}
\nabla g(\xbm) = \zerobm \quad\Leftrightarrow\quad \xbm - \gamma \nabla g(\xbm) = \xbm \;.
\end{equation*}
From the assumption $\ybm = \Abm\xbmast$ with $\xbmast \in \Zer(\Rsf)$ and from~\eqref{Sup:Eq:SetNotations}, we have the following inclusions:
\begin{equation*}
\xbmast \in \Zer(\nabla g) \cap \Zer(\Rsf) \subseteq \Fix(\Tsf) \subseteq \Im(\Dsf) \subseteq \R^n \;.
\end{equation*}
From Lemma~\ref{Sup:Lem:LipschitzDenoiser} and Lemma~\ref{Sup:Lem:ContracGradStep}, we conclude that for any $\xbm, \zbm \in \Im(\Dsf)$, we have
\begin{equation*}
\|\Tsf(\xbm)-\Tsf(\zbm)\|_2 \leq c \|\xbm-\zbm\|_2 \quad\text{with}\quad c = (1+\alpha)\max\{|1-\gamma \mu|, |1 - \gamma \lambda|\} \;.
\end{equation*}
From $\Tsf$ being a contraction over $\Im(\Dsf)$ and with Lemma~\ref{Sup:Lem:UniqueComFixedPoint}, we can conclude that $\xbmast \in \Zer(\nabla g) \cap \Zer(\Rsf)$ is the unique fixed point of PnP-PGM for any $\xbm^0 \in \Im(\Dsf)$. Thus, we have that
$$\|\xbm^t - \xbmast\|_2 = \|\Tsf(\xbm^{t-1})-\Tsf(\xbmast)\|_2 \leq c\|\xbm^{t-1}-\xbmast\|_2 \leq \cdots \leq c^t \|\xbm^0-\xbmast\|_2 \;,$$
which establishes the desired result.

\begin{figure}[t]
\centering\includegraphics[width=13.5cm]{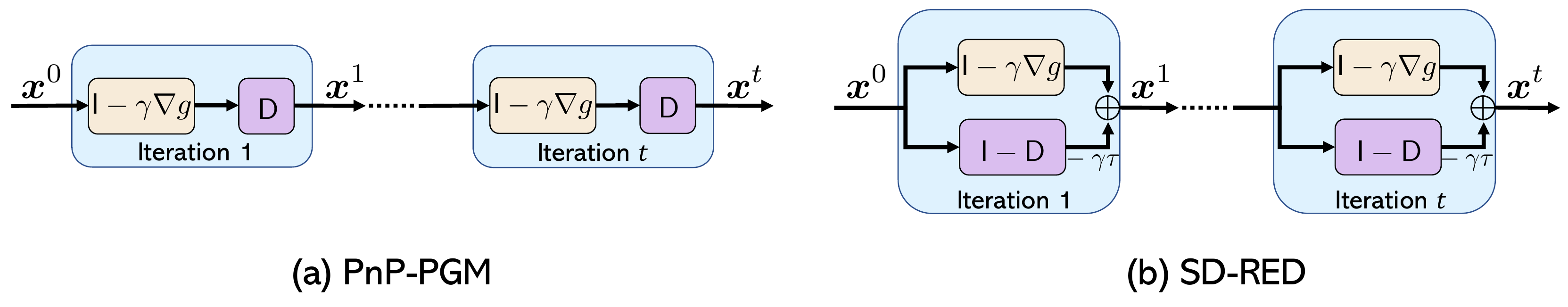}
\caption{~\emph{Algorithmic details of two optimization methods used in this work: (a) PnP-PGM and (b) SD-RED. Both algorithms are initialized with $\xbm^0$ and perform $t \geq 1$ iterations.}}
\label{Fig:PnPRED}
\end{figure}

\section{Proof of Theorem 2}
\label{Sup:Sec:Theorem2}

 In this section, we extend the analysis in Section~\ref{Sup:Sec:Theorem1} to the noisy measurement model $\ybm = \Abm\xbmast + \ebm$ where $\xbmast \in \R^n$ and $\ebm \in \R^m$. The following analysis builds on that of CSGM in~\cite{Bora.etal2017} by showing that the proof techniques used for CSGM can be also used for analyzing PnP. Note also that one can improve the error term in the recovery under an additional assumption discussed in Section~\ref{Sup:Sec:ProofLem1}.
 
 \begin{theorem}
\label{Thm:RelaxedCS}
Run PnP-PGM for $t \geq 1$ iterations under Assumptions~1-2 for the problem~(1) of the main paper with $\xbmast \in \R^n$ and $\ebm \in \R^m$. Then, the sequence $\xbm^t$ generated by PnP-PGM satisfies
\begin{equation}
\|\xbm^t - \xbmast\|_2 \leq c \|\xbm^{t-1}-\xbmast\|_2 + \varepsilon \leq c^t \|\xbm^0-\xbmast\|_2 + \frac{\varepsilon(1-c^t)}{(1-c)} \;,
\end{equation}
where $\xbm^0 \in \Im(\Dsf)$ and 
\begin{equation}
\label{Eq:PnPCSBound}
\varepsilon \defn (1+c) \left[\left(1+2\sqrt{\lambda/\mu}\right) \|\xbmast-\proj_{\Zer(\Rsf)}(\xbmast)\|_2 + 2/\sqrt{\mu} \|\ebm\|_2 + \delta(1+1/\alpha) \right]
\end{equation}
and $c \defn (1+\alpha)\max\{|1-\gamma \mu|, |1-\gamma \lambda|\}$ with $\lambda \defn \lambda_{\mathsf{max}}(\Abm^\Tsf\Abm)$.
\end{theorem}

Suppose all the assumptions for Theorem~2 are true and the step size $\gamma > 0$ is selected in a way that satisfies eq.~(10) of the main manuscript. First note that Lemma~\ref{Sup:Lem:LipschitzDenoiser} and Lemma~\ref{Sup:Lem:ContracGradStep} imply that for $\xbmbar \in \Fix(\Tsf)$, we have that
\begin{equation}
\label{Sup:Eq:PnPConv}
\|\xbm^t - \xbmbar\|_2 \leq c\|\xbm^{t-1} - \xbmbar\|_2 \quad\text{with}\quad c = (1+\alpha)\max\{|1-\gamma \mu|, |1 - \gamma \lambda|\} \in (0, 1) \;.
\end{equation}
Let $\xbmhat = \proj_{\Zer(\Rsf)}(\xbmast)$, then we have that
\begin{align*}
\|\xbmbar-\xbmhat\|
&\leq \frac{1}{\sqrt{\mu}} \left[\|\ybm-\Abm\xbmbar\|_2 + \|\ybm-\Abm\xbmhat\|_2\right] \\
&\leq \frac{1}{\sqrt{\mu}} \left[\min_{\xbm \in \Zer(\Rsf)} \|\ybm-\Abm\xbm\|_2 + \sqrt{\mu}\delta(1+1/\alpha) + \|\ybm-\Abm\xbmhat\|_2\right] \\
&\leq \frac{2}{\sqrt{\mu}} \|\ybm-\Abm\xbmhat\|_2 +\delta(1+1/\alpha) \\
&\leq 2\sqrt{\frac{\lambda}{\mu}} \|\xbmast-\xbmhat\|_2 + \frac{2}{\sqrt{\mu}}\|\ebm\|_2 + \delta(1+1/\alpha) \;,
\end{align*}
where the first inequality  uses S-REC, the second one uses Lemma~\ref{Sup:Lem:FixPointBound} in Section~\ref{Sup:Sec:ProofLem1}, the third one combines two terms by picking the larger one, and the final one uses the measurement model and the triangular inequality.
By using the inequality above, we can obtain the bound
\begin{align*}
\|\xbmbar- \xbmast\|_2 \leq \left[1+2\sqrt{\lambda/\mu}\right]\|\xbmast-\prox_{\Zer(\Rsf)}(\xbmast)\|_2 + [2/\sqrt{\mu}]\|\ebm\|_2 + \delta(1+1/\alpha) \defn \varepsilon/(1+c) \;.
\end{align*}
Note that the first two terms of $\varepsilon/(1+c)$ above are the distance of $\xbmast$ to $\Zer(\Rsf)$ and the magnitude of the error $\ebm$, and have direct analogs in standard compressed sensing. The third term is the consequence of the possibility for the solution of PnP not being in $\Zer(\Rsf)$ and as discussed in Section~\ref{Sup:Sec:ProofLem1} when $\Zer(\Rsf) \cap \Zer(\nabla g) \neq \varnothing$, then the third term disappears.

Then, from~\eqref{Sup:Eq:PnPConv}, we obtain
\begin{align*}
\|\xbm^t - \xbmast\|_2
&\leq \|\xbm^t-\xbmbar\|_2 + \|\xbmbar-\xbmast\|_2 = \|\xbm^t-\xbmbar\|_2 + \varepsilon/(c+1) \\
&\leq c\|\xbm^{t-1}-\xbmbar\|_2 + \varepsilon/(c+1)  = c\|\xbm^{t-1}-\xbmast\|_2 + c\varepsilon/(c+1) + \varepsilon/(c+1) \\
&=  c\|\xbm^{t-1}-\xbmast\|_2 + \varepsilon \leq c^t\|\xbm^0-\xbmast\|_2 + \epsilon \sum_{k = 0}^{t-1} c^k \\
&\leq c^t\|\xbm^0-\xbmast\|_2 + \varepsilon (1-c^t)/(1-c) \;,
\end{align*}
which establishes the desired result.

\subsection{A Technical Lemma for the Proof of Theorem~2}
\label{Sup:Sec:ProofLem1}

The following lemma provides a bound used for Theorem~2. As discussed within the proof, if $\Zer(\Rsf) \cap \Zer(\nabla g) \neq \varnothing$, the error term on the right of Lemma~\ref{Sup:Lem:FixPointBound} can be removed by using Lemma~\ref{Sup:Lem:UniqueComFixedPoint}. While this would lead to a tighter overall bound for Theorem~2, it would also reduce its generality. Fig.~\ref{Fig:fixconvergence} empirically shows that the sequence $\|\Rsf(\xbm^t)\|_2$ obtained by PnP-PGM in our simulations converges to a small value, suggesting that the solution obtained by the algorithm is near $\Zer(\Rsf)$.

\begin{lemma}
\label{Sup:Lem:FixPointBound}
Under the assumptions of Theorem~2 in the main manuscript, we have
\begin{equation*}
\|\ybm-\Abm\xbmbar\|_2 \leq \min_{\xbm \in \Zer(\Rsf)} \|\ybm-\Abm\xbm\|_2 + \sqrt{\mu}\delta (1+ 1/\alpha) \;.
\end{equation*}
If in addition, we know that $\Zer(\Rsf) \cap \Zer(\nabla g) \neq \varnothing$, then
\begin{equation*}
\|\ybm-\Abm\xbmbar\|_2 \leq \min_{\xbm \in \Zer(\Rsf)} \|\ybm-\Abm\xbm\|_2 \;.
\end{equation*}
\end{lemma}

\begin{proof}
First note that by re-expressing the fixed point equation of PnP-PGM, we obtain
\begin{align*}
&\xbmbar = \Dsf(\xbmbar - \gamma \nabla g(\xbmbar))  \\
&\Leftrightarrow\quad
\begin{cases}
\zbmbar = \xbmbar - \gamma \nabla g(\xbmbar) \\
\xbmbar = \zbmbar - (\zbmbar - \Dsf(\zbmbar)) = \zbmbar - \Rsf(\zbmbar)
\end{cases}
\quad\Rightarrow\quad \nabla g(\xbmbar) + \frac{1}{\gamma}\Rsf(\zbmbar) = \zerobm \;,
\end{align*}
where the final result is obtained by adding the two equalities on the left.
Since $g$ satisfies S-REC over $\Im(\Dsf)$, Lemma~\ref{Sup:Lem:EquivRSCREC} in Section~\ref{Sup:Sec:ConvexEtc} implies that for any $\xbm \in \Im(\Dsf)$ and $\xbmbar \in \Fix(\Tsf)$
\begin{align*}
g(\xbm)
&\geq g(\xbmbar) + \nabla g(\xbmbar)^\Tsf(\xbm-\xbmbar) + \frac{\mu}{2}\|\xbm-\xbmbar\|_2^2\\
&= g(\xbmbar) -(1/\gamma)  \Rsf(\zbmbar)^\Tsf(\xbm-\xbmbar) + \frac{\mu}{2}\|\xbm-\xbmbar\|_2^2\\
&\geq \min_{\xbm \in \Im(\Dsf)} \left\{g(\xbmbar) -(1/\gamma)  \Rsf(\zbmbar)^\Tsf(\xbm-\xbmbar) + \frac{\mu}{2}\|\xbm-\xbmbar\|_2^2\right\} \\
&\geq \min_{\xbm \in \R^n} \left\{g(\xbmbar) -(1/\gamma)  \Rsf(\zbmbar)^\Tsf(\xbm-\xbmbar) + \frac{\mu}{2}\|\xbm-\xbmbar\|_2^2\right\} \\
&\geq g(\xbmbar) - \frac{1}{2\mu\gamma^2}\|\Rsf(\zbmbar)\|_2^2 \;,
\end{align*}
where $\zbmbar = \xbmbar - \gamma \nabla g (\xbmbar)$. By rearranging the terms and minimizing over $\xbm \in \Im(\Dsf)$, we obtain
\begin{equation}
\label{Sup:Eq:ObjBound}
g(\xbmbar) \leq \min_{\xbm \in \Zer(\Rsf)} g(\xbm) + \frac{1}{2\mu\gamma^2} \|\Rsf(\zbmbar)\|_2^2 \leq \min_{\xbm \in \Zer(\Rsf)} g(\xbm) + \frac{\delta^2}{2\mu\gamma^2} \;,
\end{equation}
where in the last inequality we used the boundedness of $\Rsf$.

By using the actual expression of $g$ and the lower-bound on $\gamma$ in eq.~(10) in the main paper, we obtain
\begin{align*}
&1/\gamma < \mu\left(1+1/\alpha\right) \\\
&\Rightarrow\quad \|\ybm-\Abm\xbmbar\|_2 \leq \min_{\xbm \in \Zer(\Rsf)}\|\ybm-\Abm\xbm\|_2 + \delta/(\sqrt{\mu}\gamma) \leq \min_{\xbm \in \Zer(\Rsf)}\|\ybm-\Abm\xbm\|_2 + \delta\sqrt{\mu}(1+1/\alpha) \;.
\end{align*}

If we assume that $\Zer(\Rsf) \cap \Zer(\nabla g) \neq \varnothing$, then from Lemma~\ref{Sup:Lem:UniqueComFixedPoint}, we have $\xbmbar \in \Zer(\Rsf) \cap \Zer(\nabla g)$, which implies that $\xbmbar = \zbmbar$ and $\Rsf(\zbmbar) = \Rsf(\xbmbar) =  \zerobm$. In this case, we can eliminate the error term in~\eqref{Sup:Eq:ObjBound}
\begin{equation*}
g(\xbmbar) \leq \min_{\xbm \in \Zer(\Rsf)} g(\xbm) \;.
\end{equation*}
\end{proof}

\section{Proof of Theorem 3}
\label{Sup:Sec:Theorem3}

\begin{theorem}
\label{Thm:EquivREDPnP}
Suppose that Assumptions~1-3 are satisfied and that $\Zer(\nabla g) \cap \Zer(\Rsf) \neq \varnothing$, then PnP and RED have the same set of solutions: $\Fix(\Tsf) = \Zer(\Gsf)$.
\end{theorem}

The SD-RED algorithm in eq.~(6) of the main manuscript seeks zeroes of the operator
\begin{equation*}
\Gsf = \nabla g + \tau \Rsf \;.
\end{equation*}
It is clear that
\begin{equation*}
\nabla g(\zbm) = \zerobm \quad\text{and}\quad \Rsf(\zbm) = \zerobm \quad\Rightarrow\quad \Gsf(\zbm) = \zerobm \;,
\end{equation*}
which corresponds to the inclusion $\Zer(\nabla g) \cap \Zer(\Rsf) \subseteq \Zer(\Gsf)$.

We now prove the reverse inclusion under the assumptions of Theorem~3. Let $\xbm \in \Zer(\Gsf)$ and $\zbm \in \Zer(\nabla g) \cap \Zer(\Rsf)$. Since $\nabla g$ is $\lambda$-Lipschitz continuous with $\lambda = \lambda_{\textsf{\tiny max}}(\Abm^\Tsf\Abm)$, Lemma~\ref{Sup:Lem:CoCoerciveGrag} in Section~\ref{Sup:Sec:ConvexEtc} implies that $\nabla g$ is also $(1/\lambda)$-cocoercive. Therefore, we have that
\begin{equation*}
\nabla g(\xbm)^\Tsf(\xbm-\zbm) = (\nabla g(\xbm)-\nabla g(\zbm))^\Tsf(\xbm-\zbm) \geq (1/\lambda) \|\nabla g(\xbm) - \nabla g(\zbm) \|_2^2 = (1/\lambda) \|\nabla g(\xbm)\|_2^2 \;.
\end{equation*}
On the other hand, since $\Dsf$ is nonexpansive, $\Rsf = \Isf - \Dsf$ is $(1/2)$-cocoercive, which implies that
\begin{equation*}
\Rsf(\xbm)^\Tsf(\xbm-\zbm) = (\Rsf(\xbm)-\Rsf(\zbm))^\Tsf(\xbm-\zbm) \geq (1/2) \|\Rsf(\xbm) - \Rsf(\zbm) \|_2^2 = (1/2) \|\Rsf(\xbm)\|_2^2 \;.
\end{equation*}
By using the fact that $\Gsf(\xbm) = \zerobm$ and the two inequalities above, we obtain
\begin{equation}
0 = \Gsf(\xbm)^\Tsf(\xbm-\zbm) = \nabla g(\xbm)^\Tsf(\xbm-\zbm) + \tau \Rsf(\xbm)^\Tsf(\xbm-\zbm) \geq (1/\lambda) \|\nabla g(\xbm)\|_2^2 + (1/2)\|\Rsf(\xbm)\|_2^2 \;,
\end{equation}
which directly leads to the conclusion
\begin{equation*}
\Gsf(\xbm) = \zerobm \quad\Rightarrow\quad \nabla g(\xbm) = \zerobm \quad\text{and}\quad \Rsf(\xbm) = \zerobm \;.
\end{equation*}
Therefore, we have that $\Zer(\Gsf) = \Zer(\nabla g) \cap \Zer(\Rsf)$.

Note also that from Lemma~\ref{Sup:Lem:LipschitzDenoiser}, we know that when $\Zer(\nabla g) \cap \Zer(\Rsf) \neq \varnothing$, we have $\Fix(\Tsf) = \Zer(\nabla g) \cap \Zer(\Rsf)$, which directly leads to our result
$$\Zer(\Gsf) = \Zer(\nabla g) \cap \Zer(\Rsf) = \Fix(\Tsf) \;.$$

\section{Background Material}
\label{Sup:Sec:Backmaterial}

The results in this sections are well-known and can be found in different forms in standard textbooks~\cite{Bauschke.Combettes2017, Rockafellar1970, Boyd.Vandenberghe2004, Nesterov2004}. For completeness, we summarize the key results used in our analysis.

\subsection{Properties of Monotone Operators}

\medskip
\begin{definition}
An operator $\Tsf$ is Lipschitz continuous with constant $\lambda > 0$ if
\begin{equation*}
\|\Tsf(\xbm)-\Tsf(\zbm)\|_2 \leq \lambda \|\xbm-\zbm\|_2 \quad \forall \xbm, \ybm \in \R^n \;.
\end{equation*}
When $\lambda = 1$, we say that $\Tsf$ is nonexpansive. When $\lambda < 1$, we say that $\Tsf$ is a contraction.
\end{definition}

\medskip
\begin{definition}
$\Tsf$ is monotone if
\begin{equation*}
(\Tsf(\xbm)-\Tsf(\zbm))^\Tsf(\xbm-\zbm) \geq 0 \quad \forall \xbm, \ybm \in \R^n \;.
\end{equation*}
We say that $\Tsf$ is strongly monotone with parameter $\theta > 0$ if
\begin{equation*}
(\Tsf(\xbm)-\Tsf(\zbm))^\Tsf(\xbm-\zbm) \geq \theta\|\xbm-\zbm\|_2^2 \quad \forall \xbm, \ybm \in \R^n \;.
\end{equation*}
\end{definition}

\medskip
\begin{definition}
$\Tsf$ is cocoercive with constant $\beta > 0$ if
\begin{equation*}
(\Tsf(\xbm)-\Tsf(\zbm))^\Tsf(\xbm-\zbm) \geq \beta\|\Tsf(\xbm)-\Tsf(\zbm)\|_2^2 \quad \forall \xbm, \zbm \in \R^n \;.
\end{equation*}
When $\beta = 1$, we say that $\Tsf$ is firmly nonexpansive.
\end{definition}

\begin{definition}
For a constant $0 < \alpha < 1$, we say that $\Tsf$ is $\alpha$-averaged, if there exists a nonexpansive operator $\Nsf$ such that $\Tsf = (1-\alpha)\Isf + \alpha \Nsf$.
\end{definition}

\medskip\noindent
The following lemma is derived from the definitions above.

\begin{lemma}
Consider $\Rsf = \Isf - \Dsf$ where $\Dsf: \R^n \rightarrow \R^n$. Then,
\begin{equation*}
\Dsf \text{ is nonexpanisve} \quad\Leftrightarrow\quad \Rsf \text{ is } (1/2)\text{-cocoercive} \;.
\end{equation*}
\end{lemma}

\begin{proof}
First suppose that $\Rsf$ is $1/2$ cocoercive. Let $\hbm \defn \xbm - \zbm$ for any $\xbm, \zbm \in \R^n$. We then have
\begin{equation*}
\frac{1}{2}\|\Rsf(\xbm)-\Rsf(\zbm)\|_2^2 \leq (\Rsf(\xbm)-\Rsf(\zbm))^\Tsf\hbm = \|\hbm\|_2^2 - (\Dsf(\xbm)-\Dsf(\zbm))^\Tsf\hbm \;.
\end{equation*}
We also have that
\begin{equation*}
\frac{1}{2}\|\Rsf(\xbm)-\Rsf(\zbm)\|_2^2 = \frac{1}{2}\|\hbm\|^2 - (\Dsf(\xbm)-\Dsf(\zbm))^\Tsf\hbm + \frac{1}{2}\|\Dsf(\xbm)-\Dsf(\zbm)\|_2^2 \;.
\end{equation*}
By combining these two and simplifying the expression
\begin{equation*}
\|\Dsf(\xbm)-\Dsf(\zbm)\|_2 \leq \|\hbm\|_2 \;.
\end{equation*}
The converse can be proved by following this logic in reverse.
\end{proof}

\medskip\noindent
The following lemma relates the Lipschitz continuity of the residual $\Ssf = \Isf - \Tsf$ to that of $\Tsf$.
\begin{lemma}
\label{Sup:Lem:LipschitzDenoiser}
The operator $\Rsf = \Isf - \Dsf$ is $\alpha$-Lipschitz continuous if and only if the operator $(1/(1+\alpha)) \Dsf$ is nonexpansive and $\alpha/(1+\alpha)$-averaged.
\end{lemma}

\begin{proof}
See Lemma~9 in~\cite{Ryu.etal2019}.
\end{proof}

\begin{figure}[t]
\centering\includegraphics[width=13.5cm]{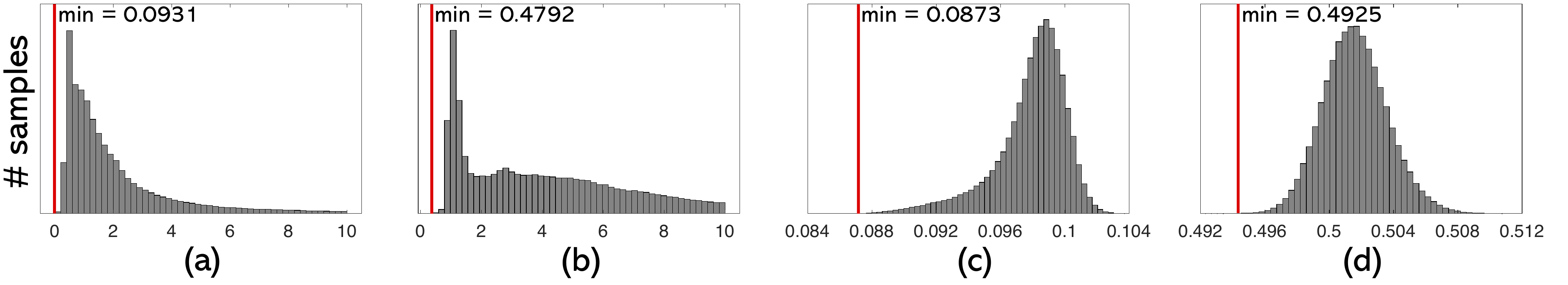}
\caption{~\emph{Empirical evaluation of the S-REC constant $\mu>0$ for the measurement operators $\Abm$ used in our simulations.  We tested both the AWGN denoisers and the AR operators by randomly sampling from their image spaces $\Im(\Dsf)$. The $x-$axis is the value of $\|\Abm\xbm - \Abm\ybm\|_2^2/\|\xbm - \ybm\|_2^2$. (a) and (b) show the histograms for the radially sub-sampled MRI matrices at $10\%$ and $50\%$ sampling ratios, respectively. (c) and (d) show the histograms for the random Gaussian matrices for the same two sampling ratios. As expected, one can observe the increase in $\mu$ for the higher sampling ratio of $50\%$.}}
\label{Fig:muest}
\end{figure}

\medskip\noindent
The following lemma considers fixed points of a composite operator.
\begin{lemma}
\label{Sup:Lem:UniqueComFixedPoint}
Let $\Tsf = \Dsf \cdot \Ssf$ with $\Fix(\Dsf) \cap \Fix(\Ssf) \neq \varnothing$ be a contraction with constant $\lambda \in (0, 1)$ over the set $\Im(\Dsf) \subseteq \R^n$. Then, we have that $\Fix(\Tsf) = \Fix(\Dsf) \cap \Fix(\Ssf)$ .
\end{lemma}

\begin{proof}
We modify the proof of Proposition 4.49 from~\cite{Bauschke.Combettes2017} to be consistent with our assumptions.

\medskip

It is clear that $\Fix(\Dsf) \cap \Fix(\Ssf) \subseteq \Fix(\Tsf)$ and our goal is to show the reverse inclusion. Let $\xbm \in \Fix(\Tsf)$ and consider three cases.
\begin{itemize}
\item \emph{Case $\Ssf(\xbm) \in \Fix(\Dsf)$}: We have that
\begin{equation*}
\Ssf(\xbm) = \Dsf(\Ssf(\xbm)) = \Tsf(\xbm) = \xbm \in \Fix(\Dsf) \cap \Fix(\Ssf) \;.
\end{equation*}

\item \emph{Case $\xbm \in \Fix(\Ssf)$}: We have that
\begin{equation*}
\Dsf(\xbm) = \Dsf(\Ssf(\xbm)) = \Tsf(\xbm) = \xbm \in \Fix(\Dsf) \cap \Fix(\Ssf) \;.
\end{equation*}

\item \emph{Case $\Ssf(\xbm) \notin \Fix(\Dsf)$ and $\xbm \notin \Fix(\Ssf)$}: Since $\Tsf = \Dsf \cdot \Ssf$ is a contraction over $\Im(\Dsf)$
\begin{equation*}
\|\xbm-\zbm\|_2 = \|\Tsf(\xbm)- \Tsf(\zbm)\|_2 \leq \lambda \|\xbm-\zbm\|_2 \quad \forall \zbm \in \Fix(\Dsf) \cap \Fix(\Ssf) \;,
\end{equation*}
which is impossible.

\end{itemize}
\end{proof}

\subsection{Convexity, restricted strong convexity, and set-restricted eigenvalue condition}
\label{Sup:Sec:ConvexEtc}

S-REC in the main manuscript can be generalized to the \emph{restricted strong convexity (RSC)} assumption, which is widely-used in the nonconvex analysis of the gradient methods (see Section~3.2 in~\cite{Jain.Kar2017}).

\begin{definition}
\label{Sup:Def:RSC}
A continuously differentiable function $g$ is said to satisfy restricted strong convexity (RSC) over $\Xcal \subseteq \R^n$ with $\mu > 0$ if
\begin{equation*}
g(\zbm) \geq g(\xbm) + \nabla g(\xbm)^\Tsf(\zbm-\xbm) + \frac{\mu}{2}\|\zbm-\xbm\|_2^2 \quad \forall \xbm, \zbm \in \Xcal \;.
\end{equation*}
\end{definition}

\begin{figure}[t]
\centering\includegraphics[width=13cm]{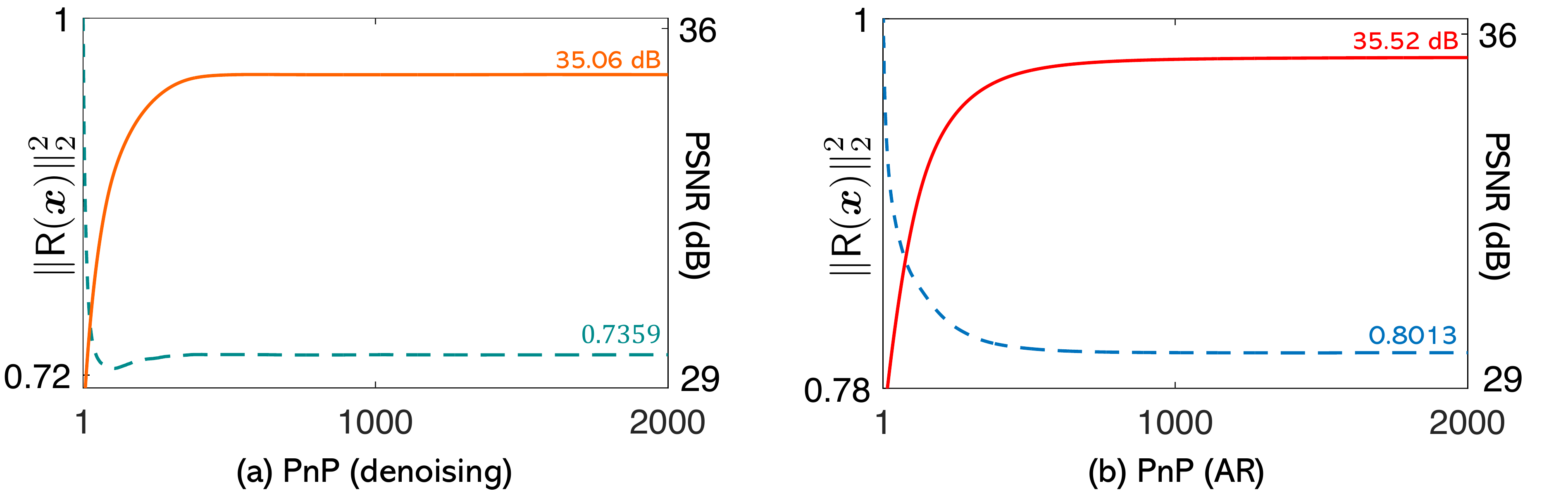}
\caption{~\emph{Illustration of the convergence of PnP under nonexpensive denoisers and AR operators. Average normalized distance to $\|\Rsf(\xbm)\|_2^2 = \|\xbm-\Dsf(\xbm)\|_2^2$ and PSNR (dB) are plotted as dashed and solid lines, respectively, against the iteration  number. This plot illustrates that PnP in our experiments converges to vectors close to $\Zer(\Rsf)$, which is consistent with the view that it regularizes inverse problems by obtaining solutions near the fixed-points of a denoiser/AR operator.}}
\label{Fig:fixconvergence}
\end{figure}

\medskip\noindent
In fact, for $g(\xbm) = \frac{1}{2}\|\ybm-\Abm\xbm\|_2^2$, S-REC is equivalent to RSC in Definition~\ref{Sup:Def:RSC}.
\begin{lemma}
\label{Sup:Lem:EquivRSCREC}
Let $g(\xbm) = \frac{1}{2}\|\ybm-\Abm\xbm\|_2^2$ and consider $\Xcal \subseteq \R^n$. Then,
\begin{equation*}
g \text{ satisfies S-REC with } \mu \text{ over } \Xcal \quad\Leftrightarrow\quad g \text{ satisfies } \mu\text{-RSC over } \Xcal \;.
\end{equation*}
\end{lemma}

\begin{proof}
Suppose $g$ is the least-squares function that satisfies S-REC with $\mu$, then for any $\xbm, \zbm \in \Xcal$
\begin{align*}
g(\zbm)
&= g(\xbm) + \nabla g(\xbm)^\Tsf(\zbm-\xbm) + \frac{1}{2}(\zbm-\xbm)\Abm^\Tsf\Abm(\zbm-\xbm) \\
&= g(\xbm) + \nabla g(\xbm)^\Tsf(\zbm-\xbm) + \frac{1}{2}\|\Abm(\zbm-\xbm)\|_2^2 \\
&\geq g(\xbm) + \nabla g(\xbm)^\Tsf(\zbm-\xbm) + \frac{\mu}{2}\|\zbm-\xbm\|_2^2 \;,
\end{align*}
which implies that $g$ satisfies $\mu$-RSC. To show S-REC using $\mu$-RSC, follow the logic in reverse.
\end{proof}

\medskip\noindent
One can use the previous and the following lemma to show that the gradient step of PnP-PGM can be a contraction for any vector in $\Im(\Dsf)$ for a properly chosen step size.

\begin{lemma}
\label{Sup:Lem:ContracGradStep}
Assume $g$ satisfies $\mu$-RSC over $\Xcal \subseteq \R^n$ and $\nabla g$ is $\lambda$-Lipschitz continuous. Then,
\begin{equation*}
\|(\Isf - \gamma \nabla g)(\xbm) - (\Isf - \gamma \nabla g)(\zbm)\|_2 \leq \max\{|1-\gamma \mu|, |1-\gamma \lambda|\}\|\xbm-\zbm\|_2 \quad\forall \xbm, \zbm \in \Xcal \;.
\end{equation*}
\end{lemma}

\begin{proof}
Since for any $\xbm, \zbm \in \Xcal$, the function $g$ is strongly convex with constant $\mu$, this lemma is a simple modification of Lemma~7 in~\cite{Ryu.etal2019}.
\end{proof}

\begin{lemma}
\label{Sup:Lem:CoCoerciveGrag}
For a convex and continuously differentiable function $g$, we have
\begin{equation*}
\nabla g \text{ is } \lambda\text{-Lipschitz continuous} \quad\Leftrightarrow\quad \nabla g \text{ is } (1/\lambda)\text{-cocoercive} \;.
\end{equation*}
\end{lemma}

\begin{proof}
See Theorem~2.1.5 in Section 2.1 of~\cite{Nesterov2004}.
\end{proof}

\section{Additional Technical Details and Numerical Results}
\label{Sec:TechnicalDetails}

We designed two types of deep priors for PnP/RED: (i) an AWGN denoiser and (ii) an artifact-removal (AR) operator trained to remove artifacts specific to the PnP iterations\footnote{The implementation of our pre-trained denoisers and AR operators are also available in the supplement.}. Both of these deep priors share the same neural network architecture, based on DnCNN~\cite{Zhang.etal2017}.  The networks contain three components. The first part is a composite convolutional layer, consisting of a normal convolutional layer and a rectified linear units (ReLU) layer. It convolves the $n_1 \times n_2$ input to $n_1 \times n_2 \times 64$ features maps by using 64 filters of size $3 \times 3$. The second part is a sequence of 10 composite convolutional layers, each having 64 filters of size $3 \times 3 \times 64$. Those composite layers further process the feature maps generated by the first part. The third part of the network, a single convolutional layer, generates the final output image by convolving the feature maps with a $3 \times 3 \times 64$ filter. Every convolution is performed with a stride $=1$, so that the intermediate feature maps share the same spatial size of the input image. We train several denoisers to optimize the \emph{mean squared error (MSE)} by using the Adam optimizer. All the experiments in this work were performed on a machine equipped with an Intel Xeon Gold 6130 Processor and eight NVIDIA GeForce RTX 2080 Ti GPUs.

\begin{figure}[t!]
\centering\includegraphics[width=13.5cm]{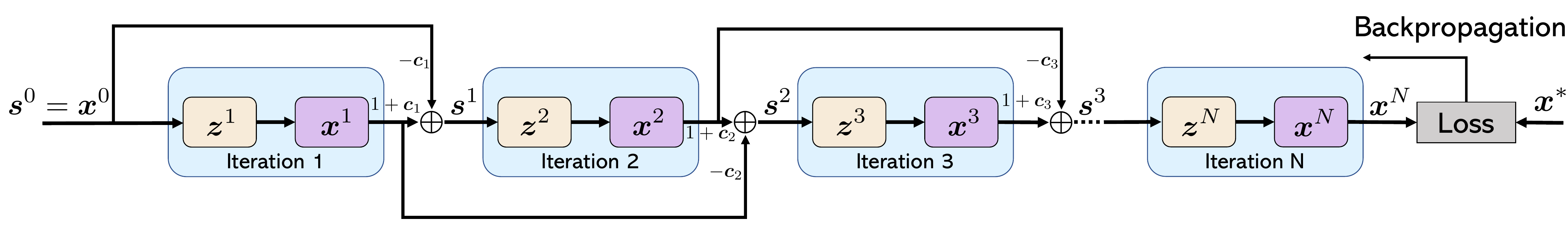}
\caption{\emph{Detailed architecture used for training the AR operator by unrolling the iterations of PnP-FISTA~\cite{Sun.etal2018a} with the DnCNN prior. Each layer contains one iteration consisting of a data-consistency update and an image prior update. The input of the unrolling network is the initialization $\xbm^0$ and the output is the reconstructed image from the $N$th iteration, which is subsequently used within the training loss. In order to make the AR operator satisfy the Assumption A, we impose the spectral normalization and weight sharing on DnCNN across different iterations. Note that once DnCNN is pre-trained following this scheme, it is used as an AR operator within PnP.}}
\label{Fig:AR}
\end{figure}

We now present the implementation details of training the AR operators used in this work.  Inspired by ISTA-Net$^+$~\footnote{ISTA-Net$^+$ is publicly available at \url{https://github.com/jianzhangcs/ISTA-Net-PyTorch}.}, we implement our own deep unfolding neural network for training the AR operator. Given an initial solution $\xbm^0$, ~\emph{i.e.} $\xbm^0 = \Abm^{\Tsf}\ybm$, we iteratively refine it by infusing information from both the gradient of the data-fidelity term $\nabla g$ and the learned operator $\Dsf$ defined as
\begin{equation}
\label{Eq:artifact estimator}
\Rsf(\xbm) = (\Isf - \Dsf)(\xbm) = \xbm - \Dsf(\xbm) \;,
\end{equation}
where $\Rsf$ is the residual of the deep neural network.
We use Nesterov acceleration in the unrolled architecture, fixing the total number of unrolling iterations to $N \geq 1$
\begin{align}
\label{Eq:PnP-Netupdate}
&\zbm^{k} = \sbm^{k-1} - \gamma \nabla g(\sbm^{k-1})\\
&\xbm^k = \Dsf(\zbm^k)\\
&c_k = (q_{k-1} -1 )/q_k\\
&\sbm^k = \xbm^k + c_k(\xbm^k - \xbm^{k-1}) \;,
\end{align}
where $\gamma>0$ is a step-size parameter and the value of $q_k = 1/2(1+\sqrt{1+4q_{k-1}^2})$ is adapted during the training for better PSNR performance.
Fig.~\ref{Fig:AR} illustrates the algorithmic details for training the AR operator. In our implementation, we opted to share the weights of the AR operator across different iterations to satisfy our theoretical assumptions. We trained several  AR operators for $N$ unfolded iterations using the MSE loss
\begin{align}
\label{Eq:lmse}
\Lcal_{\textbf{MSE}} =  \frac{1}{M}\sum_{j=1}^{M}\|\xbm_j^N - \xbm_j^\ast\|_2^2 \;,
\end{align}
where $\xbm_j^\ast$ is the ground truth. We also included a \emph{smoothness-constraint} loss across different iterations, defined as
\begin{align}
\label{Eq:lsmooth}
\Lcal_{\textbf{Smooth}} = \frac{1}{M}\sum_{j=1}^{M}\sum_{k={N-q}}^N\|\xbm_j^k - \Dsf(\zbm_j^k)\|_2^2 \;.
\end{align}
We observe that the AR operators trained with this smoothness-constraint outperform those trained without it, especially, when the AR operator is integrated into the PnP algorithm. The total AR training loss is thus $\Lcal = \Lcal_{\textbf{MSE}} + \beta\Lcal_{\textbf{Smooth}}$, where $\beta > 0$ controls the amount of smoothing. For the experiments in this paper, we set $N = 90, \beta =10$ for all gray and color AR operators' training, while we set $N=27, \beta=10$ for CS-MRI.

\begin{figure}[t]
\centering\includegraphics[width=\textwidth]{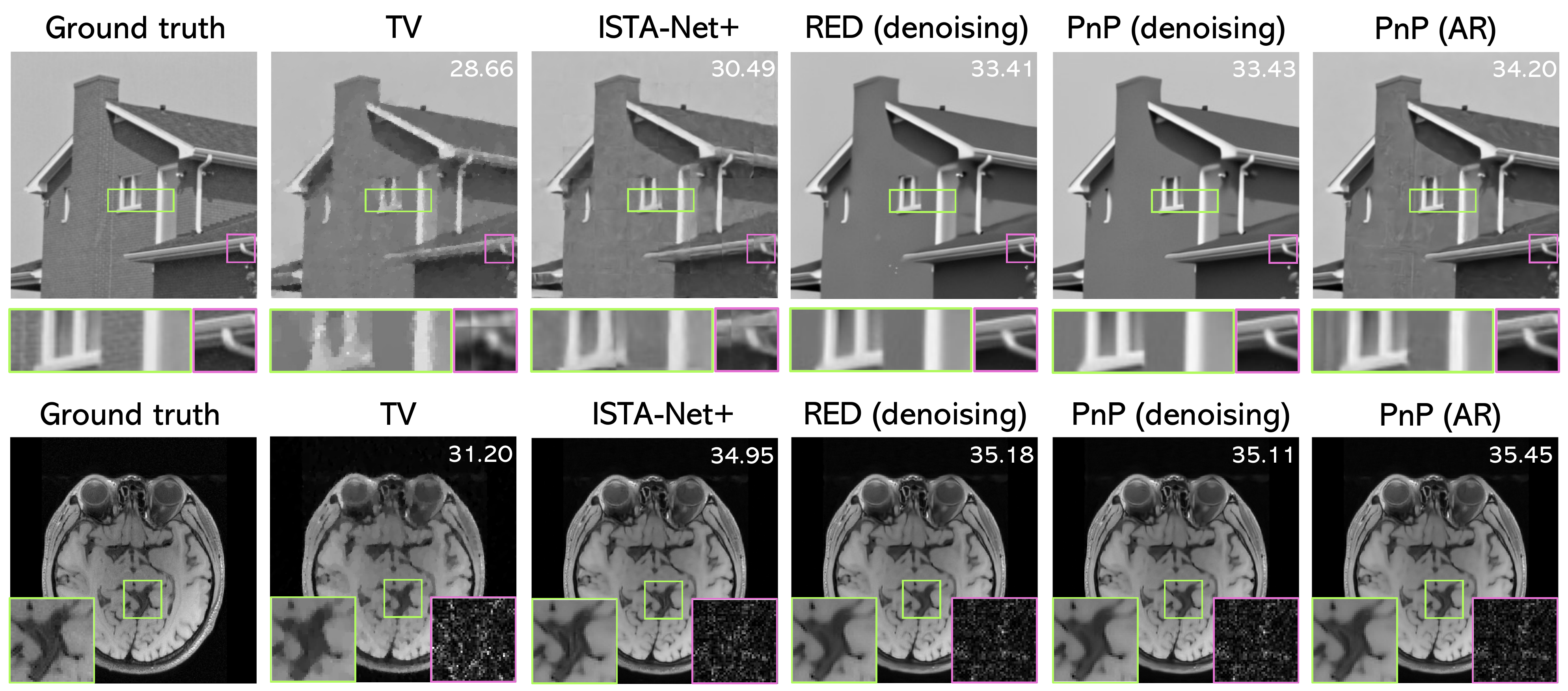}
\caption{~\emph{Additional visual comparisons between various methods for CS and CS-MRI. Top: reconstruction results on the ``House" image in Set11 at CS ratios of $10\%$. Bottom: results on MRI images with radially under-samplying at CS ratios of $20\%$ (The pink box provides the error residual that was amplified by $10\times$ for better visualization.). Best viewed by zooming in the display.}}
\label{Fig:nature_mri02}
\end{figure}

We used  a pre-training strategy to accelerate the training of the weights within the AR operator. Since the weights are shared across iterations of the deep unfolding network, we can then initialize them with those obtained from pre-trained AWGN denoisers. We observe that this pre-training strategy is considerably more efficient than initializing the entire unfolding network with the random weights. Since we initialize our learned components with deep denoisers, the initial setup for our method exactly corresponds to tuning a PnP approach with a deep denoiser.
Such training adapts the operator $\Dsf$ to a particular inverse problem and data distribution.

 \begin{table}[t]
\caption{Average PSNR (dB) values for two spectral normalization (SN) technique used in training $\alpha$-Lipschitz continuity denoisers on Set11.}
    \centering
    \renewcommand\arraystretch{1.2}
    {\footnotesize
    \scalebox{1}{
    \begin{tabular*}{13.5cm}{L{80pt}||C{130pt}lC{130pt}l}
        \hline
         \diagbox{\bf CS Ratio}{\bf Method} & \textbf{ PnP (denoiser real-SN~\cite{Ryu.etal2019})} & & \textbf{ PnP (denoiser SN~\cite{Miyato.etal2018})  }  \\ \hline\hline
         $0.1$       & 27.32  & &  27.76 \\
         $0.3$       & 34.78  & &  35.06      \\\hline
    \end{tabular*}}
    }
\label{Tab:table1}
\end{table}

In Fig.~\ref{Fig:muest}, we report the empirical evaluation of $\mu$ for the measurement operators used in our experiments by sampling images from $\Im(\Dsf)$. Specifically, we test two types of measurement matrixes for CS, namely random matrix and radially subsampled Fourier matrix, both at subsampling rates of $10\%$ and $50\%$. For each type of matrix, we first use the operator $\Dsf$ to generate several denoised image pairs on BSD68 and medical brain images, respectively. This ensures the tested image pairs are all in the range of $\Dsf$. We plot the histograms of $\mu = \|\Abm\xbm-\Abm\zbm\|_2^2 / \|\xbm-\zbm\|_2^2$, and the minimum value of each histogram is indicated by a vertical bar, providing an empirical lower bound on the values of $\mu$. Fig.~\ref{Fig:muest} illustrates that empirically the measurement operators $\Abm$ used in this work satisfies S-REC over $\Im(\Dsf)$ with $\mu > 0$.

In Table~\ref{Tab:table1}, we provide additional empirical comparisons between the spectral normalization (SN) technique in~\cite{Ryu.etal2019} and the one in~\cite{Miyato.etal2018} for training denoisers used in PnP. It is worth noting that the SN from~\cite{Miyato.etal2018} uses a convenient but inexact implementation for the convolutional layers. Both of our pre-trained models are available here: \url{https://github.com/wustl-cig/pnp-recovery}.

In Fig.~\ref{Fig:fixconvergence}, we report the convergence of $\|\Rsf(\xbm^t)\|_2^2 = \|\xbm^t - \Dsf(\xbm^t)\|_2^2$ for both the AWGN denoisers and the AR operators use in our experiments. As can be observed from the plots, in both cases, PnP converges to vectors close to $\Zer(\Rsf)$, which is consistent with the view that it regularizes inverse problems by obtaining solutions near the fixed-points of a denoiser/AR operator. Note that this view is completely backward compatible with the traditional sparsity-promoting priors and ISTA-algorithms, where one achieves regularization by promoting sparse solutions in some transform domain.

We ran fixed-point iterations of the denoisers and the AR operators used in this work on Set11. Fig.~\ref{Fig:fixpointiter} below presents visual comparisons for different values of $\|\Rsf(\xbm)\|_2^2$ for the AR operator and denoiser, respectively. Table~\ref{Tab:table2} provides PSNR (dB) for different values of $\|\Rsf\|_2^2$ using TV as a reference. In all experiments, we observed that as the images get closer to the fixed-points of the denoiser, they start losing visual details. On the other hand, deep denoisers seem to preserve visual details better than TV. This suggests that the regularization for the deep priors we used in this work is analogous to that of traditional regularization using TV, where good performance is achieved by finding images balancing data-consistency and the distance to the fixed-points of $\Dsf$ but not by returning the fixed-points of $\Dsf$ directly. Note that for some denoisers it might be desirable to directly return the images at the fixed points. For example, consider a denoiser that projects vectors to a set of natural images; the fixed-points of such denoiser are natural images.

We provide additional visualizations of the solutions produced by PnP/RED and various baseline methods considered in our work. Fig.~\ref{Fig:nature_mri02} (top) reports the visual comparison of multiple methods on Set11 with CS ratios of $10\%$, while Fig.~\ref{Fig:nature_mri02} (bottom) reports the comparison on medical brain images for CS-MRI with under-sampling ratios of $20\%$. Fig.~\ref{Fig:BSD} illustrates the numerical comparison on BSD68 for CS ratios of $30\%$ (top) and $10\%$ (bottom), respectively. Fig.~\ref{Fig:visface02} reports the visual comparison between PnP/RED and two CS methods based on StyleGAN2. Note that in all figures, PnP/RED achieves competitive results, with PnP (AR) achieving superior reconstruction results compared to PnP (denoising).

\begin{figure}[t]
\centering\includegraphics[width=0.95\textwidth]{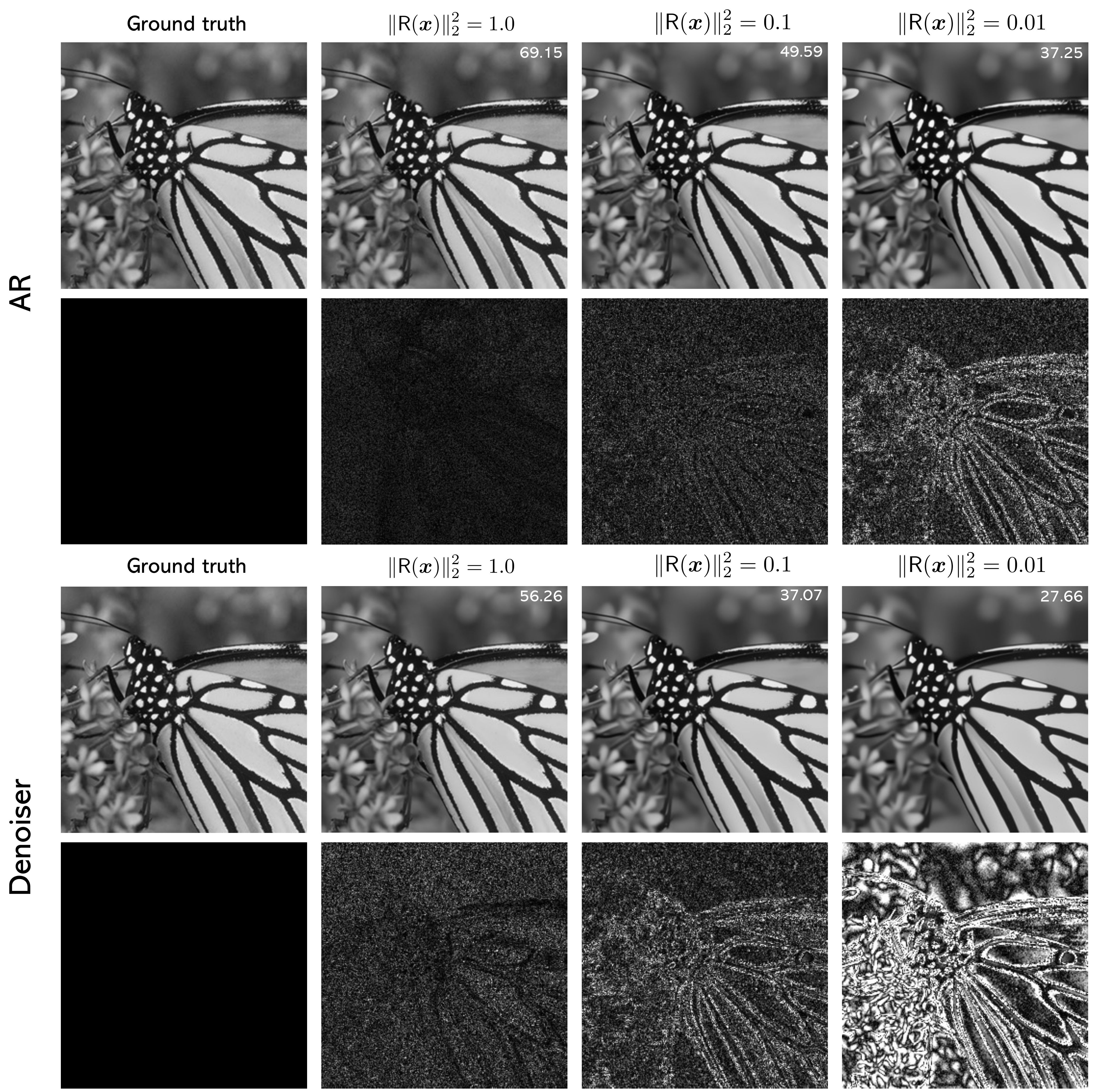}
\caption{~\emph{Visual comparison of running fixed-point iterations of the AR operators and denoisers used in the main paper when applied to the Set11. Table~\ref{Tab:table2} additionally provides PSNR (dB) for different values of $\|\Rsf{\xbm}\|_2^2$ using TV as a reference. The error residual to the ground truth images was amplified 30$\times$ and showed in grayscale for better visualization.}}
\label{Fig:fixpointiter}
\end{figure}

\begin{figure}[h!]
\centering\includegraphics[width=\textwidth]{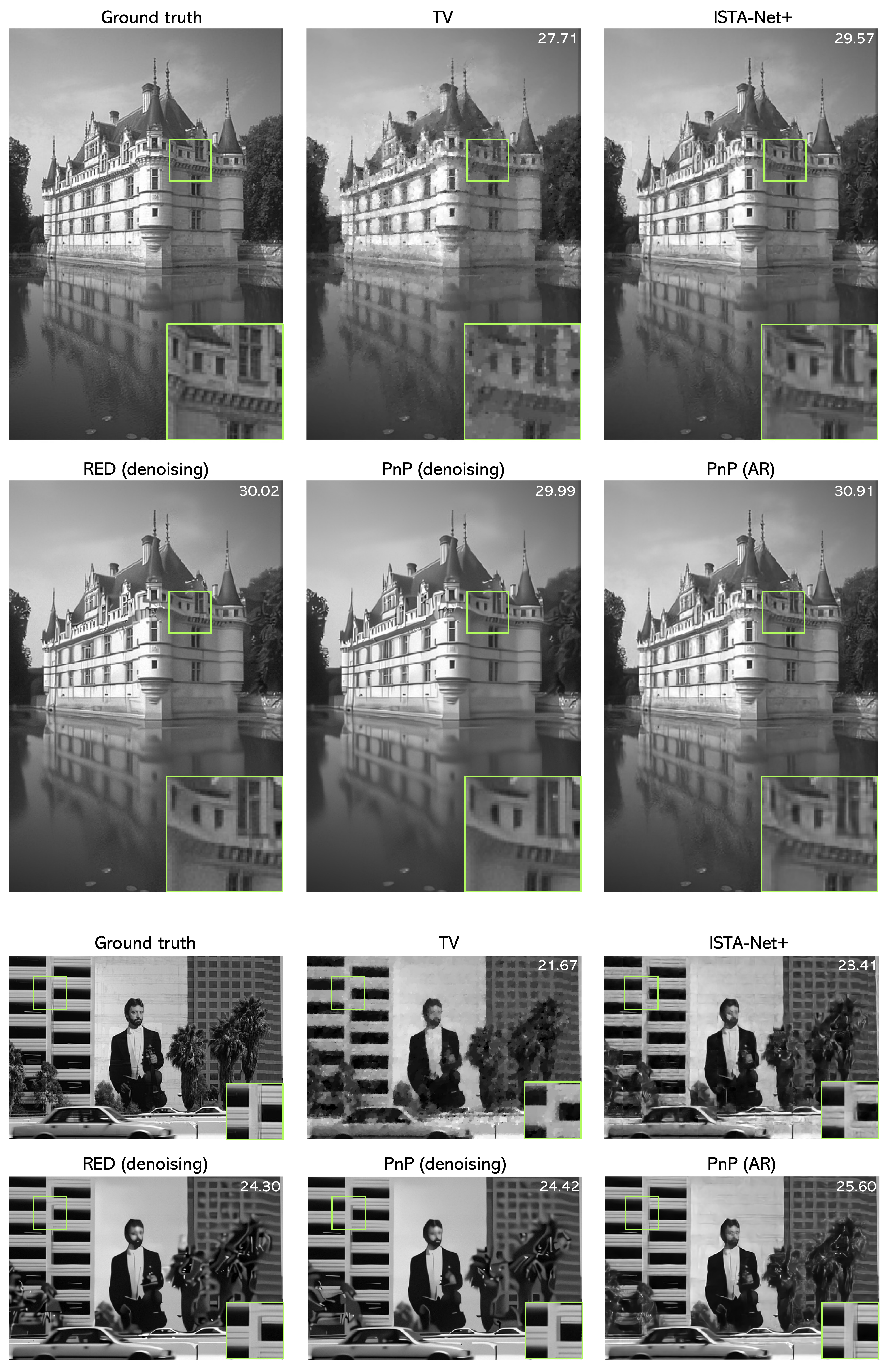}
\caption{\emph{Supplementary visual comparisons between various methods on BSD68. Top: Results at  $30\%$ sampling ratio. Bottom: Results at $10\%$ sampling ratio.}}
\label{Fig:BSD}
\end{figure}

\begin{figure}[h!]
\centering\includegraphics[width=\textwidth]{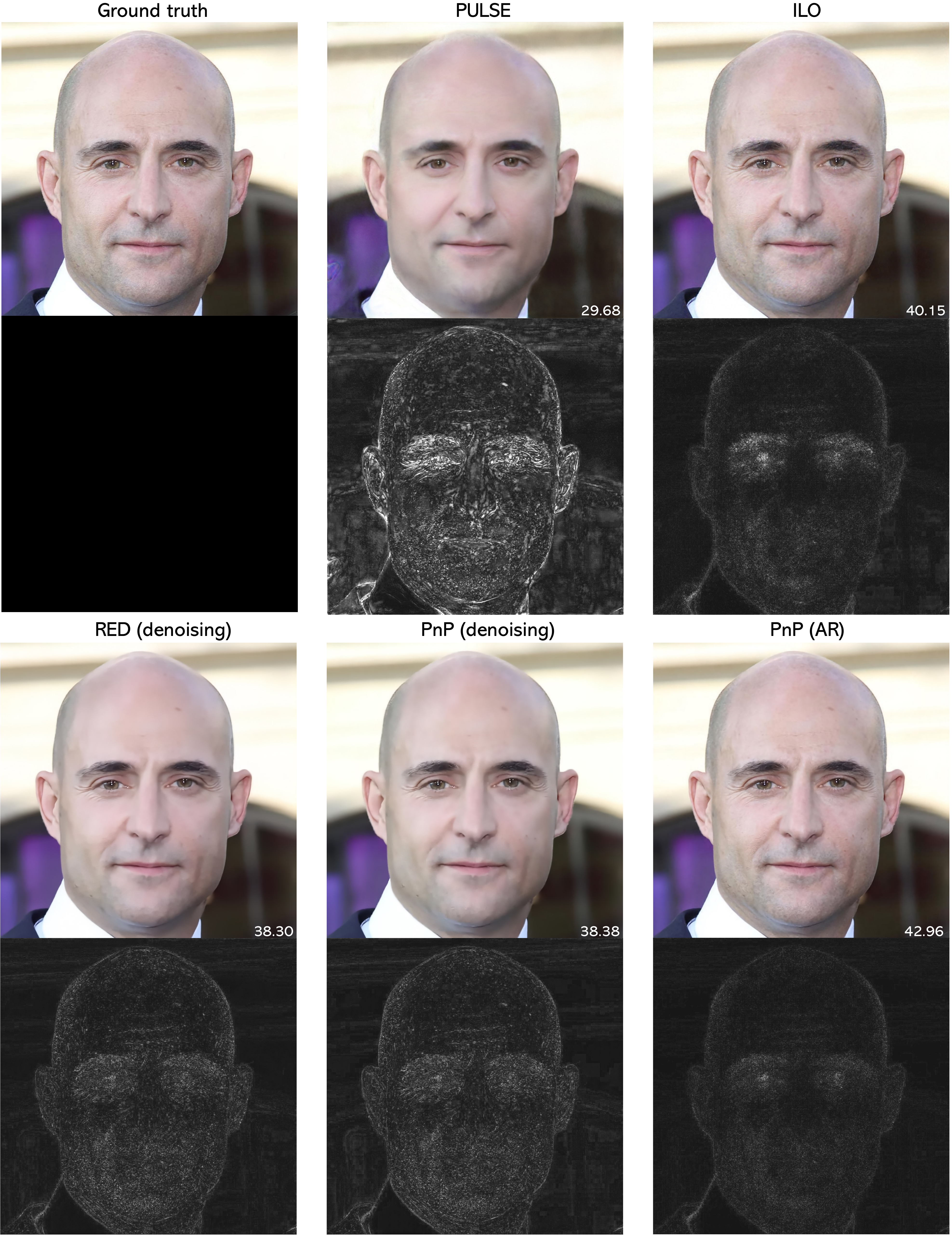}
\caption{~\emph{Visual comparison between PnP/RED and two methods using generative models, when applied to the CelebA HQ dataset at $10\%$ sampling ratio. The error residuals relative to the ground truth images were amplified 10$\times$ and showed in grayscale for better visualization. Note the similarity between the RED and PnP solutions. PnP (AR) leads to sharper images, comparable to those obtained by ILO with StyleGAN2. This highlights the benefit of using pre-trained AR operators.}}
\label{Fig:visface02}
\end{figure}

\begin{figure}[h!]
\centering\includegraphics[width=\textwidth]{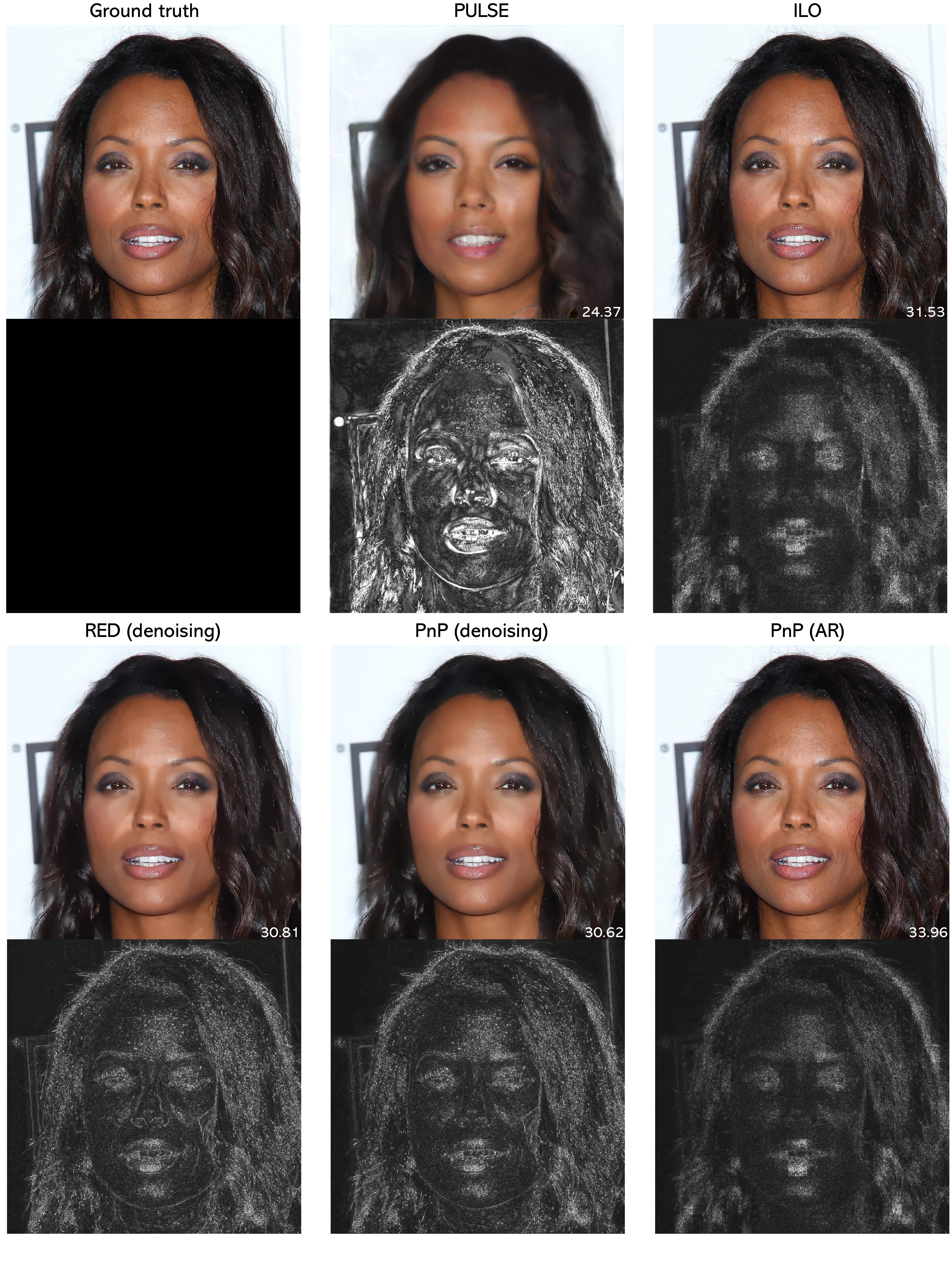}
\caption{~\emph{Visual comparison between PnP/RED and two methods using generative models, when applied to the CelebA HQ dataset at $10\%$ sampling ratio.  The error residuals to relative to the ground truth images were amplified 7$\times$ and showed in grayscale for better visualization. Note the similarity between the RED and PnP solutions. PnP (AR) leads to sharper images, comparable to those obtained by ILO with StyleGAN2. This highlights the benefit of using pre-trained AR operators.}}
\label{Fig:visface02}
\end{figure}

 \begin{table}[t]
\caption{Average PSNR (dB) values for different values of $\|\Rsf(\xbm)\|_2^2$ on Set 11.}
    \centering
    \renewcommand\arraystretch{1.2}
    {\footnotesize
    \scalebox{1}{
    \begin{tabular*}{10.5cm}{L{80pt}||C{80pt}lC{80pt}lC{80pt}|}
        \hline
         \diagbox{\bf CS Ratio}{\bf $\|\Rsf(\xbm)\|_2^2$} & \textbf{1.0} & \textbf{  0.10} & \textbf{0.001}\\ \hline\hline
         \textbf{AR}     & 72.72  &  43.98 & 35.41 \\
         \textbf{Denoiser}       & 65.76  &   35.37 & 23.81\\
         \textbf{TV}      & 14.45  & 14.24  & 13.84 \\\hline
    \end{tabular*}}
    }
\label{Tab:table2}
\end{table}


\begin{thebibliography}{10}

\bibitem{Candes.etal2006}
E.~J. Cand{\`e}s, J.~Romberg, and T.~Tao,
\newblock ``Robust uncertainty principles: Exact signal reconstruction from
  highly incomplete frequency information,''
\newblock {\em IEEE Trans. Inf. Theory}, vol. 52, no. 2, pp. 489--509, February
  2006.

\bibitem{Donoho2006}
D.~L. Donoho,
\newblock ``Compressed sensing,''
\newblock {\em IEEE Trans. Inf. Theory}, vol. 52, no. 4, pp. 1289--1306, April
  2006.

\bibitem{Bickel.etal2009}
P.~T. Bickel, R.~Ya'acov, and T.~B. Alexandre,
\newblock ``Simultaneous analysis of lasso and {D}antzig selector,''
\newblock {\em Ann. Statist.}, vol. 37, no. 4, pp. 1705--1732, 2009.

\bibitem{Wainwright2019}
M.~J. Wainwright,
\newblock {\em High-Dimensional Statistics: A Non-Asymptotic Viewpoint},
\newblock Cambridge Series in Statistical and Probabilistic Mathematics.
  Cambridge University Press, 2019.

\bibitem{Bora.etal2017}
A.~Bora, A.~Jalal, E.~Price, and A.~G. Dimakis,
\newblock ``Compressed sensing using generative priors,''
\newblock in {\em Proc. 34th Int. Conf. Machine Learning ({ICML})}, Sydney,
  Australia, Aug. 2017, pp. 537--546.

\bibitem{Shah.Hegde2018}
V.~Shah and C.~Hegde,
\newblock ``Solving linear inverse problems using {GAN} priors: {A}n algorithm
  with provable guarantees,''
\newblock in {\em Proc. IEEE Int. Conf. Acoustics, Speech and Signal Process.},
  Calgary, AB, Canada, Apr. 2018, pp. 4609--4613.

\bibitem{Hegde2018}
C.~Hegde,
\newblock ``Algorithmic aspects of inverse problems using generative models,''
\newblock in {\em Proc. Allerton Conf. Communication, Control, and Computing},
  Monticellu, IL, USA, Oct. 2018, pp. 166--172.

\bibitem{Latorre.etal2019}
F.~Latorre, A.~Eftekhari, and V.~Cevher,
\newblock ``Fast and provable {ADMM} for learning with generative priors,''
\newblock in {\em Advances in Neural Information Processing Systems 33},
  Vancouver, BC, USA, December 8-14, 2019, pp. 12027--12039.

\bibitem{Hyder.etal2019}
R.~Hyder, V.~Shah, C.~Hegde, and M.~S. Asif,
\newblock ``Alternating phase projected gradient descent with generative priors
  for solving compressive phase retrieval,''
\newblock in {\em Proc. IEEE Int. Conf. Acoustics, Speech and Signal Process.},
  Brighton, UK, May 2019, pp. 7705--7709.

\bibitem{Venkatakrishnan.etal2013}
S.~V. Venkatakrishnan, C.~A. Bouman, and B.~Wohlberg,
\newblock ``Plug-and-play priors for model based reconstruction,''
\newblock in {\em Proc. IEEE Global Conf. Signal Process. and Inf. Process.
  ({GlobalSIP})}, Austin, TX, USA, December 3-5, 2013, pp. 945--948.

\bibitem{Sreehari.etal2016}
S.~Sreehari, S.~V. Venkatakrishnan, B.~Wohlberg, G.~T. Buzzard, L.~F. Drummy,
  J.~P. Simmons, and C.~A. Bouman,
\newblock ``Plug-and-play priors for bright field electron tomography and
  sparse interpolation,''
\newblock {\em IEEE Trans. Comp. Imag.}, vol. 2, no. 4, pp. 408--423, December
  2016.

\bibitem{Romano.etal2017}
Y.~Romano, M.~Elad, and P.~Milanfar,
\newblock ``The little engine that could: {R}egularization by denoising
  ({RED}),''
\newblock {\em SIAM J. Imaging Sci.}, vol. 10, no. 4, pp. 1804--1844, 2017.

\bibitem{Zhang.etal2017a}
K.~Zhang, W.~Zuo, S.~Gu, and L.~Zhang,
\newblock ``Learning deep {CNN} denoiser prior for image restoration,''
\newblock in {\em Proc. {IEEE} Conf. Computer Vision and Pattern Recognition
  ({CVPR})}, Honolulu, HI, USA, July 2017, pp. 2808--2817.

\bibitem{Metzler.etal2018}
C.~A. Metzler, P.~Schniter, A.~Veeraraghavan, and R.~G. Baraniuk,
\newblock ``pr{D}eep: {R}obust phase retrieval with a flexible deep network,''
\newblock in {\em Proc. 35th Int. Conf. Machine Learning ({ICML})}, Stockholm,
  Sweden, June 2018, pp. 3501--3510.

\bibitem{Dong.etal2019}
W.~Dong, P.~Wang, W.~Yin, G.~Shi, F.~Wu, and X.~Lu,
\newblock ``Denoising prior driven deep neural network for image restoration,''
\newblock {\em IEEE Trans. Patt. Anal. and Machine Intell.}, vol. 41, no. 10,
  pp. 2305--2318, Oct. 2019.

\bibitem{Zhang.etal2019}
K.~Zhang, W.~Zuo, and L.~Zhang,
\newblock ``Deep plug-and-play super-resolution for arbitrary blur kernels,''
\newblock in {\em Proc. {IEEE} Conf. Computer Vision and Pattern Recognition
  ({CVPR})}, Long Beach, CA, USA, June 16-20, 2019, pp. 1671--1681.

\bibitem{Ahmad.etal2020}
R.~{Ahmad}, C.~A. {Bouman}, G.~T. {Buzzard}, S.~{Chan}, S.~{Liu}, E.~T.
  {Reehorst}, and P.~{Schniter},
\newblock ``Plug-and-play methods for magnetic resonance imaging: Using
  denoisers for image recovery,''
\newblock {\em IEEE Signal Processing Magazine}, vol. 37, no. 1, pp. 105--116,
  2020.

\bibitem{Wei.etal2020}
K.~Wei, A.~Aviles-Rivero, J.~Liang, Y.~Fu, C.-B. Sch\"onlieb, and H.~Huang,
\newblock ``Tuning-free plug-and-play proximal algorithm for inverse imaging
  problems,''
\newblock in {\em Proc. 37th Int. Conf. Machine Learning ({ICML})}, 2020.

\bibitem{Chan.etal2016}
S.~H. Chan, X.~Wang, and O.~A. Elgendy,
\newblock ``Plug-and-play {ADMM} for image restoration: Fixed-point convergence
  and applications,''
\newblock {\em IEEE Trans. Comp. Imag.}, vol. 3, no. 1, pp. 84--98, March 2017.

\bibitem{Meinhardt.etal2017}
T.~Meinhardt, M.~Moeller, C.~Hazirbas, and D.~Cremers,
\newblock ``Learning proximal operators: {U}sing denoising networks for
  regularizing inverse imaging problems,''
\newblock in {\em Proc. IEEE Int. Conf. Comp. Vis. (ICCV)}, Venice, Italy, Oct.
  2017, pp. 1799--1808.

\bibitem{Buzzard.etal2017}
G.~T. Buzzard, S.~H. Chan, S.~Sreehari, and C.~A. Bouman,
\newblock ``Plug-and-play unplugged: {O}ptimization free reconstruction using
  consensus equilibrium,''
\newblock {\em SIAM J. Imaging Sci.}, vol. 11, no. 3, pp. 2001--2020, 2018.

\bibitem{Reehorst.Schniter2019}
E.~T. Reehorst and P.~Schniter,
\newblock ``Regularization by denoising: {C}larifications and new
  interpretations,''
\newblock {\em IEEE Trans. Comput. Imag.}, vol. 5, no. 1, pp. 52--67, Mar.
  2019.

\bibitem{Ryu.etal2019}
E.~K. Ryu, J.~Liu, S.~Wnag, X.~Chen, Z.~Wang, and W.~Yin,
\newblock ``Plug-and-play methods provably converge with properly trained
  denoisers,''
\newblock in {\em Proc. 36th Int. Conf. Machine Learning (ICML)}, Long Beach,
  CA, USA, June 2019, pp. 5546--5557.

\bibitem{Sun.etal2018a}
Y.~Sun, B.~Wohlberg, and U.~S. Kamilov,
\newblock ``An online plug-and-play algorithm for regularized image
  reconstruction,''
\newblock {\em IEEE Trans. Comput. Imag.}, vol. 5, no. 3, pp. 395--408, Sept.
  2019.

\bibitem{Tirer.Giryes2019}
T.~Tirer and R.~Giryes,
\newblock ``Image restoration by iterative denoising and backward
  projections,''
\newblock {\em IEEE Trans. Image Process.}, vol. 28, no. 3, pp. 1220--1234,
  2019.

\bibitem{Teodoro.etal2019}
A.~M. Teodoro, J.~M. Bioucas-Dias, and {M. A. T.} Figueiredo,
\newblock ``A convergent image fusion algorithm using scene-adapted
  {G}aussian-mixture-based denoising,''
\newblock {\em IEEE Trans. Image Process.}, vol. 28, no. 1, pp. 451--463, Jan.
  2019.

\bibitem{Xu.etal2020}
X.~Xu, Y.~Sun, J.~Liu, B.~Wohlberg, and U.~S. Kamilov,
\newblock ``Provable convergence of plug-and-play priors with {MMSE}
  denoisers,''
\newblock {\em IEEE Signal Process. Lett.}, vol. 27, pp. 1280--1284, 2020.

\bibitem{Sun.etal2021}
Y.~Sun, Z.~Wu, B.~Wohlberg, and U.~S. Kamilov,
\newblock ``Scalable plug-and-play {ADMM} with convergence guarantees,''
\newblock {\em IEEE Trans. Comput. Imag.}, vol. 7, pp. 849--863, July 2021.

\bibitem{Sun.etal2019b}
Y.~Sun, J.~Liu, and U.~S. Kamilov,
\newblock ``Block coordinate regularization by denoising,''
\newblock in {\em Proc. Advances in Neural Information Processing Systems 33},
  Vancouver, BC, Canada, Dec. 2019, pp. 382--392.

\bibitem{Cohen.etal2020}
R.~Cohen, M.~Elad, and P.~Milanfar,
\newblock ``Regularization by denoising via fixed-point projection (red-pro),''
\newblock {\em SIAM Journal on Imaging Sciences}, vol. 14, no. 3, pp.
  1374--1406, 2021.

\bibitem{Rudin.etal1992}
L.~I. Rudin, S.~Osher, and E.~Fatemi,
\newblock ``Nonlinear total variation based noise removal algorithms,''
\newblock {\em Physica D}, vol. 60, no. 1--4, pp. 259--268, November 1992.

\bibitem{Bioucas-Dias.Figueiredo2007}
J.~M. Bioucas-Dias and M.~A.~T. Figueiredo,
\newblock ``A new {T}w{IST}: {T}wo-step iterative shrinkage/thresholding
  algorithms for image restoration,''
\newblock {\em IEEE Trans. Image Process.}, vol. 16, no. 12, pp. 2992--3004,
  December 2007.

\bibitem{Beck.Teboulle2009a}
A.~Beck and M.~Teboulle,
\newblock ``Fast gradient-based algorithm for constrained total variation image
  denoising and deblurring problems,''
\newblock {\em IEEE Trans. Image Process.}, vol. 18, no. 11, pp. 2419--2434,
  November 2009.

\bibitem{Karras.etal2019}
T.~Karras, S.~Laine, and T.~Aila,
\newblock ``A style-based generator architecture for generative adversarial
  networks,''
\newblock in {\em Proc. {IEEE} Conf. Computer Vision and Pattern Recognition
  ({CVPR})}, Long Beach, CA, USA, June 2019, pp. 4396--4405.

\bibitem{Karras.etal2020}
T.~Karras, S.~Laine, M.~Aittala, J.~Hellsten, J.~Lehtinen, and T.~Aila,
\newblock ``Analyzing and improving the image quality of {StyleGAN},''
\newblock in {\em Proc. {IEEE} Conf. Computer Vision and Pattern Recognition
  ({CVPR})}, Seattle, WA, USA, June 2020, pp. 8107--8116.

\bibitem{Menon.etal2020}
S.~Menon, A.~Damian, S.~Hu, N.~Ravi, and C.~Rudin,
\newblock ``{PULSE}: Self-supervised photo upsampling via latent space
  exploration of generative models,''
\newblock in {\em Proc. {IEEE} Conf. Computer Vision and Pattern Recognition
  ({CVPR})}, Seattle, WA, USA, June 2020, pp. 2434--2442.

\bibitem{Daras.etal2021}
G.~Daras, J.~Dean, A.~Jalal, and A.~G. Dimakis,
\newblock ``Intermediate layer optimizationfor inverse problems using deep
  generative models,''
\newblock 2021,
\newblock arXiv:2102.07364.

\bibitem{Candes.Wakin2008}
E.~Cand{\`e}s and M.~B. Wakin,
\newblock ``An introduction to compressive sensing,''
\newblock {\em IEEE Signal Process. Mag.}, vol. 25, no. 2, pp. 21--30, March
  2008.

\bibitem{Parikh.Boyd2014}
N.~Parikh and S.~Boyd,
\newblock ``Proximal algorithms,''
\newblock {\em Foundations and Trends in Optimization}, vol. 1, no. 3, pp.
  123--231, 2014.

\bibitem{Dabov.etal2007}
K.~Dabov, A.~Foi, V.~Katkovnik, and K.~Egiazarian,
\newblock ``Image denoising by sparse {3-D} transform-domain collaborative
  filtering,''
\newblock {\em IEEE Trans. Image Process.}, vol. 16, no. 16, pp. 2080--2095,
  August 2007.

\bibitem{Zhang.etal2017}
K.~Zhang, W.~Zuo, Y.~Chen, D.~Meng, and L.~Zhang,
\newblock ``Beyond a {G}aussian denoiser: {R}esidual learning of deep {CNN} for
  image denoising,''
\newblock {\em IEEE Trans. Image Process.}, vol. 26, no. 7, pp. 3142--3155,
  July 2017.

\bibitem{Figueiredo.Nowak2003}
M.~A.~T. Figueiredo and R.~D. Nowak,
\newblock ``An {EM} algorithm for wavelet-based image restoration,''
\newblock {\em IEEE Trans. Image Process.}, vol. 12, no. 8, pp. 906--916,
  August 2003.

\bibitem{Daubechies.etal2004}
I.~Daubechies, M.~Defrise, and C.~De Mol,
\newblock ``An iterative thresholding algorithm for linear inverse problems
  with a sparsity constraint,''
\newblock {\em Commun. Pure Appl. Math.}, vol. 57, no. 11, pp. 1413--1457,
  November 2004.

\bibitem{Bect.etal2004}
J.~Bect, L.~Blanc-Feraud, G.~Aubert, and A.~Chambolle,
\newblock ``A $\ell_1$-unified variational framework for image restoration,''
\newblock in {\em Proc. {ECCV}}, Springer, Ed., New York, 2004, vol. 3024, pp.
  1--13.

\bibitem{Beck.Teboulle2009}
A.~Beck and M.~Teboulle,
\newblock ``A fast iterative shrinkage-thresholding algorithm for linear
  inverse problems,''
\newblock {\em SIAM J. Imaging Sciences}, vol. 2, no. 1, pp. 183--202, 2009.

\bibitem{Kamilov.etal2017}
U.~S. Kamilov, H.~Mansour, and B.~Wohlberg,
\newblock ``A plug-and-play priors approach for solving nonlinear imaging
  inverse problems,''
\newblock {\em IEEE Signal Process. Lett.}, vol. 24, no. 12, pp. 1872--1876,
  December 2017.

\bibitem{Mataev.etal2019}
G.~Mataev, M.~Elad, and P.~Milanfar,
\newblock ``{DeepRED}: {D}eep image prior powered by {RED},''
\newblock in {\em Proc. {IEEE} Int. Conf. Comp. Vis. Workshops ({ICCVW})},
  Seoul, South Korea, Oct 27-Nov 2, 2019, pp. 1--10.

\bibitem{Liu.etal2020}
J.~Liu, Y.~Sun, C.~Eldeniz, W.~Gan, H.~An, and U.~S. Kamilov,
\newblock ``{RARE}: {I}mage reconstruction using deep priors learned without
  ground truth,''
\newblock {\em IEEE J. Sel. Topics Signal Process.}, vol. 14, no. 6, pp.
  1088--1099, 2020.

\bibitem{Xie.etal2021}
M.~Xie, J.~Liu, Y.~Sun, W.~Gan, B.~Wohlberg, and U.~S. Kamilov,
\newblock ``Joint reconstruction and calibration using regularization by
  denoising with application to computed tomography,''
\newblock in {\em Proc. {IEEE} Int. Conf. Comp. Vis. Workshops ({ICCVW})},
  October 2021, pp. 4028--4037.

\bibitem{McCann.etal2017}
M.~T. McCann, K.~H. Jin, and M.~Unser,
\newblock ``Convolutional neural networks for inverse problems in imaging: A
  review,''
\newblock {\em IEEE Signal Process. Mag.}, vol. 34, no. 6, pp. 85--95, 2017.

\bibitem{Lucas.etal2018}
A.~Lucas, M.~Iliadis, R.~Molina, and A.~K. Katsaggelos,
\newblock ``Using deep neural networks for inverse problems in imaging:
  {B}eyond analytical methods,''
\newblock {\em IEEE Signal Process. Mag.}, vol. 35, no. 1, pp. 20--36, Jan.
  2018.

\bibitem{Knoll.etal2020}
F.~Knoll, K.~Hammernik, C.~Zhang, S.~Moeller, T.~Pock, D.~K. Sodickson, and
  M.~Akcakaya,
\newblock ``Deep-learning methods for parallel magnetic resonance imaging
  reconstruction: {A} survey of the current approaches, trends, and issues,''
\newblock {\em IEEE Signal Process. Mag.}, vol. 37, no. 1, pp. 128--140, Jan.
  2020.

\bibitem{Ongie.etal2020}
G.~Ongie, A.~Jalal, C.~A. Metzler, R.~G. Baraniuk, A.~G. Dimakis, and
  R.~Willett,
\newblock ``Deep learning techniques for inverse problems in imaging,''
\newblock {\em IEEE J. Sel. Areas Inf. Theory}, vol. 1, no. 1, pp. 39--56, May
  2020.

\bibitem{Tan.etal2015}
J.~Tan, Y.~Ma, and D.~Baron,
\newblock ``Compressive imaging via approximate message passing with image
  denoising,''
\newblock {\em IEEE Trans. Signal Process.}, vol. 63, no. 8, pp. 2085--2092,
  Apr. 2015.

\bibitem{Metzler.etal2016}
C.~A. Metzler, A.~Maleki, and R.~G. Baraniuk,
\newblock ``From denoising to compressed sensing,''
\newblock {\em IEEE Trans. Inf. Theory}, vol. 62, no. 9, pp. 5117--5144,
  September 2016.

\bibitem{Metzler.etal2016a}
C.~A. Metzler, A.~Maleki, and R.~Baraniuk,
\newblock ``{BM3D}-{PRGAMP}: {C}ompressive phase retrieval based on {BM3D}
  denoising,''
\newblock in {\em Proc. {IEEE} Int. Conf. Image Proc.}, Phoenix, AZ, USA,
  September 25-28, 2016, pp. 2504--2508.

\bibitem{Fletcher.etal2018}
A.~Fletcher, S.~Rangan, S.~Sarkar, and P.~Schniter,
\newblock ``Plug-in estimation in high-dimensional linear inverse problems: {A}
  rigorous analysis,''
\newblock in {\em Proc. Advances in Neural Information Processing Systems 32},
  Montr\'eal, Canada, Dec 3-8, 2018, pp. 7451--7460.

\bibitem{Gregor.LeCun2010}
K.~Gregor and Y.~LeCun,
\newblock ``Learning fast approximation of sparse coding,''
\newblock in {\em Proc. 27th Int. Conf. Machine Learning (ICML)}, Haifa,
  Israel, June 21-24, 2010, pp. 399--406.

\bibitem{Schmidt.Roth2014}
U.~Schmidt and S.~Roth,
\newblock ``Shrinkage fields for effective image restoration,''
\newblock in {\em Proc. {IEEE} Conf. Computer Vision and Pattern Recognition
  ({CVPR})}, Columbus, OH, USA, June 23-28, 2014, pp. 2774--2781.

\bibitem{Chen.etal2015}
Y.~Chen, W.~Yu, and T.~Pock,
\newblock ``On learning optimized reaction diffuction processes for effective
  image restoration,''
\newblock in {\em Proc. {IEEE} Conf. Computer Vision and Pattern Recognition
  ({CVPR})}, Boston, MA, USA, June 8-10, 2015, pp. 5261--5269.

\bibitem{Yang.etal2016}
Y.~Yang, J.~Sun, H.~Li, and Z.~Xu,
\newblock ``Deep {ADMM}-{N}et for compressive sensing {MRI},''
\newblock in {\em Proc. Advances in Neural Information Processing Systems 29},
  2016, pp. 10--18.

\bibitem{zhang2018ista}
J.~{Zhang} and B.~{Ghanem},
\newblock ``{ISTA-Net}: {I}nterpretable optimization-inspired deep network for
  image compressive sensing,''
\newblock in {\em Proc. {IEEE} Conf. Computer Vision and Pattern Recognition
  ({CVPR})}, 2018, pp. 1828--1837.

\bibitem{Aggarwal.etal2019}
H.~K. {Aggarwal}, M.~P. {Mani}, and M.~{Jacob},
\newblock ``{MoDL}: {M}odel-based deep learning architecture for inverse
  problems,''
\newblock {\em IEEE Trans. Med. Imag.}, vol. 38, no. 2, pp. 394--405, Feb.
  2019.

\bibitem{Liu.etal2021}
J.~Liu, Y.~Sun, W.~Gan, B.~Wohlberg, and U.~S. Kamilov,
\newblock ``{SGD-Net}: {E}fficient model-based feep learning with theoretical
  guarantees,''
\newblock {\em IEEE Trans. Comput. Imaging}, vol. 7, pp. 598--610, 2021.

\bibitem{Ulyanov.etal2018}
D.~Ulyanov, A.~Vedaldi, and V.~Lempitsky,
\newblock ``Deep image prior,''
\newblock in {\em Proc. {IEEE} Conf. Computer Vision and Pattern Recognition
  ({CVPR})}, Salt Lake City, UT, USA, June 18-22, 2018, pp. 9446--9454.

\bibitem{heckel2018deep}
R.~Heckel and P.~Hand,
\newblock ``Deep decoder: {C}oncise image representations from untrained
  non-convolutional networks,''
\newblock in {\em Int. Conf. on Learning Representations (ICLR)}, 2018.

\bibitem{Terris.etal2020}
M.~Terris, A.~Repetti, J.-C. Pesquet, and Y.~Wiaux,
\newblock ``Building firmly nonexpansive convolutional neural networks,''
\newblock in {\em Proc. IEEE Int. Conf. Acoustics, Speech and Signal Process.},
  Barcelona, Spain, May 2020, pp. 8658--8662.

\bibitem{Miyato.etal2018}
T.~Miyato, T.~Kataoka, M.~Koyama, and Y.~Yoshida,
\newblock ``Spectral normalization for generative adversarial networks,''
\newblock in {\em Int. Conf. on Learning Representations ({ICLR})}, Vancouver,
  Canada, Apr. 2018.

\bibitem{Fazlyab.etal2019}
M.~Fazlyab, A.~Robey, Hassani. H., M.~Marari, and G.~Pappas,
\newblock ``Efficient and accurate estimation of {L}ipschitz constants for deep
  neural networks,''
\newblock in {\em Proc. Advances in Neural Information Processing Systems 33},
  Vancouver, BC, Canada, Dec. 2019, pp. 11427--11438.

\bibitem{Elad2010}
M.~Elad,
\newblock {\em Sparse and Redundant Representations},
\newblock Springer, 2010.

\bibitem{Foucart.Rauhut2013}
S.~Foucart and H.~Rauhut,
\newblock {\em A Methematical Introduction to Compressive Sensing},
\newblock Birkhauser, 2013.

\bibitem{Vershynin2018}
R.~Vershynin,
\newblock {\em High-Dimensional Probability: {A}n Introduction with
  Applications in Data Science},
\newblock Cambridge Series in Statistical and Probabilistic Mathematics.
  Cambridge University Press, 2018.

\bibitem{Bauschke.Combettes2017}
H.~H. Bauschke and P.~L. Combettes,
\newblock {\em Convex Analysis and Monotone Operator Theory in Hilbert Spaces},
\newblock Springer, 2 edition, 2017.

\bibitem{Gilton.etal2021}
D.~Gilton, G.~Ongie, and R.~Willett,
\newblock ``Deep equilibrium architectures for inverse problems in imaging,''
\newblock 2021,
\newblock arXiv:2102.07944.

\bibitem{Xu.etal2021}
X.~Xu, J.~Liu, Y.~Sun, B.~Wohlberg, and U.S. Kamilov,
\newblock ``Boosting the performance of plug-and-play priors via denoiser
  scaling,''
\newblock in {\em 2020 54th Asilomar Conference on Signals, Systems, and
  Computers}. IEEE, 2020, pp. 1305--1312.

\bibitem{Mousavi.etal2015}
A.~Mousavi, A.~B. Patel, and R.~G. Baraniuk,
\newblock ``A deep learning approach to structured signal recovery,''
\newblock in {\em Proc. Allerton Conf. Communication, Control, and Computing},
  Allerton Park, IL, USA, September 30-October 2, 2015, pp. 1336--1343.

\bibitem{Kulkarni.etal2016}
K.~Kulkarni, S.~Lohit, P.~Turaga, R.~Kerviche, and A.~Ashok,
\newblock ``{ReconNet}: {N}on-iterative reconstruction of images from
  compressively sensed measurements,''
\newblock in {\em Proc. {IEEE} Conf. Computer Vision and Pattern Recognition
  ({CVPR})}, Las Vegas, NV, USA, June 2016, pp. 449--458.

\bibitem{Martin.etal2001}
D.~Martin, C.~Fowlkes, D.~Tal, and J.~Malik,
\newblock ``A database of human segmented natural images and its application to
  evaluating segmentation algorithms and measuring ecological statistics,''
\newblock in {\em Proc. {IEEE} Int. Conf. Comp. Vis. ({ICCV})}, Vancouver,
  Canada, July 7-14, 2001, pp. 416--423.

\bibitem{Agustsson.etal2017}
E.~Agustsson and R.~Timofte,
\newblock ``Ntire 2017 challenge on single image super-resolution: Dataset and
  study,''
\newblock in {\em Proc. {IEEE} Conf. Computer Vision and Pattern Recognition
  ({CVPR}) Workshops}, July 2017, vol.~3, pp. 126--135.

\bibitem{Lustig.etal2007}
M.~Lustig, D.~L. Donoho, and J.~M. Pauly,
\newblock ``Sparse {MRI}: The application of compressed sensing for rapid {MR}
  imaging,''
\newblock {\em Magn. Reson. Med.}, vol. 58, no. 6, pp. 1182--1195, December
  2007.

\bibitem{Lustig.etal2008}
M.~Lustig, D.~L. Donoho, J.~M. Santos, and J.~M. Pauly,
\newblock ``Compressed sensing {MRI},''
\newblock {\em IEEE Signal Process. Mag.}, vol. 25, no. 2, pp. 72--82, 2008.

\bibitem{karras.etal2018}
T.~Karras, T.~Aila, S.~Laine, and J.~Lehtinen,
\newblock ``Progressive growing of {GAN}s for improved quality, stability, and
  variation,''
\newblock in {\em Int. Conf. on Learning Representations (ICLR)}, 2018.

\bibitem{Nesterov2004}
Y.~Nesterov,
\newblock {\em Introductory Lectures on Convex Optimization: A Basic Course},
\newblock Kluwer Academic Publishers, 2004.

\bibitem{Eslahi.Foi2018}
N.~Eslahi and A.~Foi,
\newblock ``Anisotropic spatiotemporal regularization in compressive video
  recovery by adaptively modeling the residual errors as correlated noise,''
\newblock in {\em IEEE Image, Video, and Multidimensional Signal Processing
  Workshop}, Zagorochoria, Greece, June 2018, pp. 1--5.

\bibitem{Bauschke.Combettes2017}
H.~H. Bauschke and P.~L. Combettes,
\newblock {\em Convex Analysis and Monotone Operator Theory in Hilbert Spaces},
\newblock Springer, 2 edition, 2017.

\bibitem{Ryu.Boyd2016}
E.~K. Ryu and S.~Boyd,
\newblock ``A primer on monotone operator methods,''
\newblock {\em Appl. Comput. Math.}, vol. 15, no. 1, pp. 3--43, 2016.

\bibitem{Rockafellar1970}
R.~T. Rockafellar,
\newblock {\em Convex Analysis},
\newblock Princeton Univ. Press, Princeton, NJ, 1970.

\bibitem{Boyd.Vandenberghe2004}
S.~Boyd and L.~Vandenberghe,
\newblock {\em Convex Optimization},
\newblock Cambridge Univ. Press, 2004.

\bibitem{Jain.Kar2017}
P.~Jain and P.~Kar,
\newblock ``Non-convex optimization for machine learning,''
\newblock {\em Foundations and Trends in Machine Learning}, vol. 10, no. 3-4,
  pp. 142--363, 2017.

\end{thebibliography}
\end{document}